\documentclass[11pt]{article}
\usepackage{times}

 \usepackage[letterpaper, left=1in, right=1in, top=1in, bottom=1in]{geometry}

\usepackage{parskip}

\usepackage[dvipsnames,table]{xcolor}
\usepackage[colorlinks=true, linkcolor=blue, citecolor=blue]{hyperref}
\usepackage{microtype}

\usepackage[utf8]{inputenc} 
\usepackage[T1]{fontenc}    
\usepackage{url}            
\usepackage{booktabs}       
\usepackage{amsfonts}       
\usepackage{nicefrac}       

\usepackage{natbib}
\bibpunct{(}{)}{;}{a}{,}{,}

\usepackage{amsthm}
\usepackage{mathtools}
\usepackage{amsmath}
\usepackage{bbm}
\usepackage{amsfonts}
\usepackage{amssymb}

\usepackage{MnSymbol} 

\usepackage{xpatch}

\theoremstyle{definition}  

\newtheorem{lemma}{Lemma}

\newtheorem{proposition}{Proposition}
\newtheorem{fact}{Fact}
\newtheorem{assumption}{Assumption}

\theoremstyle{plain}

\newtheorem{theorem}{Theorem}
\newtheorem{definition}{Definition}

\xpatchcmd{\proof}{\itshape}{\normalfont\proofnameformat}{}{}
\newcommand{\proofnameformat}{\bfseries}

\usepackage{prettyref}
\newcommand{\pref}[1]{\prettyref{#1}}
\newcommand{\pfref}[1]{Proof of \prettyref{#1}}
\newcommand{\savehyperref}[2]{\texorpdfstring{\hyperref[#1]{#2}}{#2}}
\newrefformat{eq}{\savehyperref{#1}{\textup{(\ref*{#1})}}}
\newrefformat{eqn}{\savehyperref{#1}{Equation~\ref*{#1}}}
\newrefformat{tab}{\savehyperref{#1}{Table~\ref*{#1}}}
\newrefformat{lem}{\savehyperref{#1}{Lemma~\ref*{#1}}}
\newrefformat{def}{\savehyperref{#1}{Definition~\ref*{#1}}}
\newrefformat{thm}{\savehyperref{#1}{Theorem~\ref*{#1}}}
\newrefformat{fact}{\savehyperref{#1}{Fact~\ref*{#1}}}
\newrefformat{corr}{\savehyperref{#1}{Corollary~\ref*{#1}}}
\newrefformat{sec}{\savehyperref{#1}{Section~\ref*{#1}}}
\newrefformat{app}{\savehyperref{#1}{Appendix~\ref*{#1}}}
\newrefformat{ass}{\savehyperref{#1}{Assumption~\ref*{#1}}}
\newrefformat{ex}{\savehyperref{#1}{Example~\ref*{#1}}}
\newrefformat{fig}{\savehyperref{#1}{Figure~\ref*{#1}}}
\newrefformat{alg}{\savehyperref{#1}{Algorithm~\ref*{#1}}}
\newrefformat{rem}{\savehyperref{#1}{Remark~\ref*{#1}}}
\newrefformat{conj}{\savehyperref{#1}{Conjecture~\ref*{#1}}}
\newrefformat{prop}{\savehyperref{#1}{Proposition~\ref*{#1}}}
\newrefformat{proto}{\savehyperref{#1}{Protocol~\ref*{#1}}}
\newrefformat{prob}{\savehyperref{#1}{Problem~\ref*{#1}}}
\newrefformat{claim}{\savehyperref{#1}{Claim~\ref*{#1}}}

\DeclarePairedDelimiter{\abs}{\lvert}{\rvert} %
\DeclarePairedDelimiter{\brk}{[}{]}
\DeclarePairedDelimiter{\crl}{\{}{\}}
\DeclarePairedDelimiter{\prn}{(}{)}
\DeclarePairedDelimiter{\nrm}{\|}{\|}
\DeclarePairedDelimiter{\tri}{\langle}{\rangle}

\DeclarePairedDelimiter{\ceil}{\lceil}{\rceil}

\DeclareMathOperator{\En}{\mathbb{E}}

\DeclareMathOperator*{\argmin}{arg\,min} 

\newcommand{\ls}{\ell}
\newcommand{\ind}{\mathbbm{1}}    
\newcommand{\pmo}{\crl*{\pm{}1}}
\newcommand{\eps}{\epsilon}
\newcommand{\veps}{\varepsilon}

\newcommand{\ldef}{\vcentcolon=}
\newcommand{\rdef}{=\vcentcolon}

\newcommand{\mb}[1]{\boldsymbol{#1}}

\newcommand{\wt}[1]{\widetilde{#1}}
\newcommand{\wh}[1]{\widehat{#1}}

\def\ddefloop#1{\ifx\ddefloop#1\else\ddef{#1}\expandafter\ddefloop\fi}
\def\ddef#1{\expandafter\def\csname bb#1\endcsname{\ensuremath{\mathbb{#1}}}}
\ddefloop ABCDEFGHIJKLMNOPQRSTUVWXYZ\ddefloop
\def\ddefloop#1{\ifx\ddefloop#1\else\ddef{#1}\expandafter\ddefloop\fi}
\def\ddef#1{\expandafter\def\csname b#1\endcsname{\ensuremath{\mathbf{#1}}}}
\ddefloop ABCDEFGHIJKLMNOPQRSTUVWXYZ\ddefloop
\def\ddef#1{\expandafter\def\csname c#1\endcsname{\ensuremath{\mathcal{#1}}}}
\ddefloop ABCDEFGHIJKLMNOPQRSTUVWXYZ\ddefloop
\def\ddef#1{\expandafter\def\csname h#1\endcsname{\ensuremath{\widehat{#1}}}}
\ddefloop ABCDEFGHIJKLMNOPQRSTUVWXYZ\ddefloop
\def\ddef#1{\expandafter\def\csname hc#1\endcsname{\ensuremath{\widehat{\mathcal{#1}}}}}
\ddefloop ABCDEFGHIJKLMNOPQRSTUVWXYZ\ddefloop
\def\ddef#1{\expandafter\def\csname t#1\endcsname{\ensuremath{\widetilde{#1}}}}
\ddefloop ABCDEFGHIJKLMNOPQRSTUVWXYZ\ddefloop
\def\ddef#1{\expandafter\def\csname tc#1\endcsname{\ensuremath{\widetilde{\mathcal{#1}}}}}
\ddefloop ABCDEFGHIJKLMNOPQRSTUVWXYZ\ddefloop

\newcommand{\Holder}{H{\"o}lder}

\usepackage{enumitem}
\usepackage{xspace}
\usepackage{hhline}
\usepackage{amssymb}

\usepackage[scaled=.9]{helvet}

\newcommand{\sgn}{\mathrm{sgn}}

\newcommand{\xr}[1][n]{x_{1:#1}}

\newcommand{\yr}[1][n]{y_{1:#1}}
\newcommand{\zr}[1][n]{z_{1:#1}}

\newcommand{\W}{\mathcal{W}}

\newcommand{\Bspace}{\mathfrak{B}}

\newcommand{\grad}{\nabla}

\newcommand{\trn}{\intercal}

\renewcommand{\trn}{\dagger}

\newcommand{\poprisk}{L_{\cD}}
\newcommand{\emprisk}{\wh{L}_{n}}

\newcommand{\midsem}{\,;}

\newtheorem{manualtheoreminner}{Theorem}
\newenvironment{manualtheorem}[1]{
  \renewcommand\themanualtheoreminner{#1}
  \manualtheoreminner
}{\endmanualtheoreminner}

\newcommand\numberthis{\addtocounter{equation}{1}\tag{\theequation}}

\renewcommand{\trn}{\top}

\newcommand{\glmtron}{\textsf{GLMtron}\xspace}

\newcommand{\pl}{Polyak-\L{}ojasiewicz\xspace}

\newcommand{\kl}{Gradient Domination\xspace}
\newcommand{\klshort}{GD\xspace}

\newcommand{\est}{\wh{w}}
\newcommand{\walg}{\wh{w}^{\mathrm{alg}}}
\newcommand{\wopt}{w^{\star}}
\newcommand{\optrisk}{L^{\star}}

\newcommand{\tightoverset}[2]{\mathop{#2}\limits^{\vbox to -.35ex{\kern-0.75ex\hbox{$#1$}\vss}}}

\newcommand{\rad}{\mathfrak{R}}

\newcommand{\radvec}{\tightoverset{\rightarrow}{\rad}}
\newcommand{\radnorm}{\rad_{\nrm*{\cdot}}}

\newcommand{\beps}{\mb{\eps}}

\renewcommand{\tW}{\wt{\cW}}

\newcommand{\eigmin}{\lambda_{\mathrm{min}}}
\newcommand{\eigminr}{\psi_{\mathrm{min}}}

\newcommand{\muhigh}{\mu_{\text{h}}}
\newcommand{\Chigh}{C_{\text{h}}}
\newcommand{\mulow}{\mu_{\text{l}}}
\newcommand{\Clow}{C_{\text{l}}}
\newcommand{\musparse}{\mu_{\text{s}}}
\newcommand{\Csparse}{C_{\text{s}}}
\newcommand{\tgamma}{\tilde{\gamma}}

\newcommand{\ones}{\mb{1}}
\newcommand{\vphi}{\varphi}

\newcommand{\softmargin}{$\phi$-soft-margin\xspace}

\newcommand{\empdist}{\wh{\cD}_n}

\usepackage{color-edits}
\addauthor{df}{red}
\addauthor{as}{orange}
\addauthor{ks}{blue}

\title{\huge{Uniform Convergence of Gradients for\\  Non-Convex Learning and Optimization}}

\author{  
   Dylan J. Foster\qquad{}Ayush Sekhari\qquad{}Karthik Sridharan\\~\\
     {\normalsize \texttt{\{djfoster, sekhari, sridharan\}@cs.cornell.edu}}
}

\date{}

\begin{document}

\maketitle

\begin{abstract}
We investigate 1) the rate at which refined properties of the empirical risk---in particular, gradients---converge to their population counterparts in standard non-convex learning tasks, and 2) the consequences of this convergence for optimization. Our analysis follows the tradition of \emph{norm-based capacity control}. We propose vector-valued Rademacher complexities as a simple, composable, and user-friendly tool to derive \emph{dimension-free} uniform convergence bounds for gradients in non-convex learning problems.
As an application of our techniques, we give a new analysis of batch gradient descent methods for non-convex generalized linear models and non-convex robust regression, showing how to use any algorithm that finds approximate stationary points to obtain optimal sample complexity, even when dimension is high or possibly infinite and multiple passes over the dataset are allowed.

Moving to non-smooth models we show----in contrast to the smooth case---that even for a single ReLU it is not possible to obtain dimension-independent convergence rates for gradients in the worst case. On the positive side, it is still possible to obtain dimension-independent rates under a new type of distributional assumption.
\end{abstract}

\section{Introduction}
\label{sec:intro}

The last decade has seen a string of empirical successes for gradient-based algorithms solving large scale non-convex machine learning problems \citep{krizhevsky2012imagenet, he2016deep}. Inspired by these successes, the theory community has begun to make progress on understanding when gradient-based methods succeed for non-convex learning in certain settings of interest \citep{jain2017non}. The goal of the present work is to introduce learning-theoretic tools to---in a general sense---improve understanding of when and why gradient-based methods succeed for non-convex learning problems.

In a standard formulation of the non-convex statistical learning problem, we aim to solve
\[
\textrm{minimize} \quad \poprisk(w)\ldef{} \En_{(x,y)\sim{}\cD}\ls(w\midsem{}x,y),
\]
where $w\in\cW\subseteq{}\bbR^{d}$ is a parameter vector, $\cD$ is an unknown probability distribution over the instance space $\cX\times{}\cY$, and the loss $\ls$ is a potentially non-convex function of $w$. The learner cannot observe $\cD$ directly and instead must find a model $\est\in\cW$ that minimizes $\poprisk$ given only access to i.i.d. samples $(x_1,y_1),\ldots,(x_n,y_n)\sim{}\cD$. Their performance is quantified by the \emph{excess risk} $\poprisk(\est)-\optrisk$, where $\optrisk=\inf_{w\in\cW}\poprisk(w)$.

Given only access to samples, a standard (``sample average approximation'') approach is to attempt to minimize the \emph{empirical risk} $\emprisk(w)\ldef{}\frac{1}{n}\sum_{t=1}^{n}\ls(w\midsem{}x_t,y_t)$.
If one succeeds at minimizing $\emprisk$, classical statistical learning theory provides a comprehensive set of tools to bounds the excess risk of the procedure. The caveat is that when $\ls$ is non-convex, global optimization of the empirical risk may be far from easy. It is not typically viable to even \emph{verify} whether one is at a global minimizer of $\emprisk$. Moreover, even if the population risk $\poprisk$ has favorable properties that make it amenable to gradient-based optimization, the empirical risk may not inherit these properties due to stochasticity. In the worst case, minimizing $\poprisk$ or $\emprisk$ is simply intractable. However, recent years have seen a number of successes showing that for non-convex problems arising in machine learning, iterative optimizers can succeed both in theory and in practice (see \cite{jain2017non} for a survey). Notably, while minimizing $\emprisk$ might be challenging, there is an abundance of gradient methods that provably find approximate stationary points of the empirical risk, i.e. $\nrm*{\grad\emprisk(w)}\leq{}\veps$ \citep{nesterov2013introductory,ghadimi2013stochastic,reddi2016stochastic,allen2016variance,lei2017non}.
 In view of this, the present work has two aims: {\em First}, to provide a general set of tools to prove uniform convergence results for gradients
 , with the goal of bounding how many samples are required to ensure that with high probability over samples, simultaneously over all $w \in \W$, $\|\nabla \poprisk(w)\| \le \|\nabla \emprisk(w)\| + \veps$;~~{\em Second,} to explore concrete non-convex problems where one can establish that the excess risk $\poprisk(\est)-\optrisk$ is small a consequence of this gradient uniform convergence. Together, these two directions yield direct bounds on the convergence of non-convex gradient-based learning algorithms to low excess risk.

Our precise technical contributions are as follows:
\begin{itemize}
\item We bring vector-valued Rademacher complexities \citep{maurer2016vector} and associated vector-valued contraction principles to bear on the analysis of uniform convergence for gradients.  This approach enables \emph{norm-based capacity control}, meaning that the bounds are independent of dimension whenever the predictor norm and data norm are appropriately controlled. We introduce a ``chain rule'' for Rademacher complexity, which enables one to decompose the complexity of gradients of compositions into complexities of their components, and makes deriving dimension-independent complexity bounds for common non-convex classes quite simple.

\item We establish variants of the \emph{\kl{}} condition for the population risk in certain non-convex learning settings. The condition bounds excess risk in terms of the magnitude of gradients, and is satisfied in non-convex learning problems including generalized linear models and robust regression. As a consequence of the gradient uniform convergence bounds, we show how to use any algorithm that finds approximate stationary points for smooth functions in a black-box fashion to obtain optimal sample complexity for these models---both in high- and low-dimensional regimes. In particular, standard algorithms including gradient descent \citep{nesterov2013introductory}, SGD \citep{ghadimi2013stochastic}, Non-convex SVRG \citep{reddi2016stochastic,allen2016variance}, and SCSG \citep{lei2017non} enjoy optimal sample complexity, even when allowed to take multiple passes over the dataset.

\item  We show that for non-smooth losses dimension-independent uniform convergence is not possible in the worst case, but that this can be circumvented using a new type of margin assumption.
\end{itemize}

\paragraph{Related Work} 
This work is inspired by \cite{mei2016landscape}, who gave dimension-dependent gradient and Hessian convergence rates and optimization guarantees for the generalized linear model and robust regression setups we study. We move beyond the dimension-dependent setting by providing norm-based capacity control. Our bounds are independent of dimension whenever the predictor norm and data norm are sufficiently controlled (they work in infinite dimension in the $\ls_2$ case), but even when the norms are large we recover the optimal dimension-dependent rates.

Optimizing the empirical risk under assumptions on the population risk has begun to attract significant attention (e.g. \cite{gonen2017fast,jin2018minimizing}). Without attempting a complete survey, we remark that these results typically depend on dimension, e.g. \cite{jin2018minimizing} require $\mathrm{poly}(d)$ samples before their optimization guarantees take effect. We view these works as complementary to our norm-based analysis.
\paragraph{Notation}
For a given norm $\nrm*{\cdot}$, the dual norm is denoted $\nrm*{\cdot}_{\star}$. $\nrm*{\cdot}_{p}$ represents the standard $\ls_{p}$ norm on $\bbR^{d}$ and  $\nrm*{\cdot}_{\sigma}$ denotes the spectral norm. $\ones$ denotes the all-ones vector, with dimension made clear from context. For a function $f:\bbR^{d}\to\bbR$, $\grad{}f(x) \in \bbR^d$  and  $\grad^{2}f(x)\in\bbR^{d\times{}d}$ will denote the gradient and the Hessian of $f$ at $x$ respectively.  $f$ is said to be $L$-Lipschitz with respect to a norm $\nrm*{\cdot}$ if $\abs*{f(x)-f(y)}\leq{}L\nrm*{x-y}\;\forall{}x,y$. Similarly, $f$  is said to be H-smooth w.r.t norm $\nrm*{\cdot}$ if its gradients are H-Lipschitz with respect to $\nrm*{\cdot}$, i.e. $\nrm*{\grad{}f(x)-\grad{}f(y)}_{\star}\leq{}H\nrm*{x-y}$ for some $H$. 


\section{Gradient Uniform Convergence: Why and How}
\label{sec:tools}

\subsection{Utility of Gradient Convergence: The Why}
Before introducing our tools for establishing gradient uniform convergence, let us introduce a family of losses for which this convergence has immediate consequences for the design of non-convex statistical learning algorithms.
\begin{definition}
The population risk $\poprisk$ satisfies the $(\alpha,\mu)$-\kl{} condition with respect to a norm $\nrm*{\cdot}$ if there are constants $\mu>0$, $\alpha\in[1,2]$ such that
\begin{equation}
\label{eq:kl}\tag{\klshort}
\poprisk(w) - \poprisk(w^{\star}) \leq{} \mu\nrm*{\grad{}\poprisk(w)}^{\alpha} \quad\forall{}w\in\cW,
\end{equation}
where $w^{\star}\in\argmin{}_{w\in\cW}\poprisk(w)$ is any population risk minimizer.
\end{definition}

The case $\alpha=2$ is often referred to as the \pl{} inequality \citep{polyak1963gradient,karimi2016linear}. The general \klshort{} condition implies that all critical points are global, and is itself implied (under technical restrictions) by many other well-known conditions including one-point convexity \citep{li2017convergence}, star convexity and $\tau$-star convexity \citep{hardt2016gradient}, and so-called ``regularity conditions'' \citep{zhang2016median}; for more see \cite{karimi2016linear}. The \klshort{} condition is satsified---sometimes locally rather than globally, and usually under distributional assumptions--- by the population risk in settings including neural networks with one hidden layer \citep{li2017convergence}, ResNets with linear activations \citep{hardt2016identity}, phase retrieval \citep{zhang2016median}, matrix factorization \citep{liu2016quadratic}, blind deconvolution \citep{li2016rapid}, and---as we show here---generalized linear models and robust regression.

The \klshort{} condition states that to optimize the population risk it suffices to find a (population) stationary point. What are the consequences of the statement for the learning problem, given that the learner only has access to the empirical risk $\emprisk{}$ which itself may not satisfy \klshort{}? The next proposition shows, via gradient uniform convergence, that \klshort{} is immediately useful for non-convex learning even when it is only satisfied at the population level. \\

\begin{proposition}
\label{prop:kl_unif}
Suppose that $\poprisk$ satisfies the $(\alpha,\mu)$-\klshort{} condition. Then, for any $\delta > 0$, with probability at least $1 - \delta$ over the draw of the data $\crl*{(x_t,y_t)}_{t=1}^{n}$, every algorithm $\walg$ satisfies
\begin{equation}
\label{eq:kl_rademacher}
\poprisk(\walg) - \optrisk \leq{} 2\ \mu\prn*{\nrm*{\grad\emprisk(\walg)}^{\alpha} + \En\sup_{w \in \cW}\nrm*{\grad\emprisk(w) - \grad\poprisk(w)}^{\alpha} +  c \left(\frac{\log\left(\frac{1}{\delta}\right)}{n}\right)^{\frac{\alpha}{2}} } ~,
\end{equation}
where the constant $c$ depends only on the range of $\nrm*{\grad{}\ls}$.
\end{proposition}

Note that if $\cW$ is a finite set, then standard concentration arguments for norms along with the union bound imply that $\En\sup_{w \in \cW}\nrm*{\grad\emprisk(w) - \grad\poprisk(w)} \le O\left(\sqrt{\frac{\log|\cW|}{n}} \right)$. For smooth losses, if $\cW \subset \mathbb{R}^d$ is contained in a bounded ball, then by simply discretizing the set $\cW$ up to precision $\veps$ (with $O(\veps^{-d}$) elements), one can easily obtain a bound of $\En\sup_{w \in \cW}\nrm*{\grad\emprisk(w) - \grad\poprisk(w)} \le O\prn*{\sqrt{\frac{d}{n}}}$. This approach recovers the dimension-dependent gradient convergence rates obtained in \cite{mei2016landscape}.

Our goal is to go beyond this type of analysis and provide dimension-free rates that apply even when the dimension is larger than the number of examples, or possibly infinite. Our bounds take the following ``norm-based capacity control'' form: $\En\sup_{w \in \cW}\nrm*{\grad\emprisk(w) - \grad\poprisk(w)} \le O\left(\sqrt{\frac{\mathcal{C}(\cW)}{n}} \right)$ where $\mathcal{C}(\cW)$ is a norm-dependent, but dimension-independent measure of the size of $\cW$. Given such a bound, any algorithm that guarantees $\nrm*{\grad\emprisk(\walg)}\leq{}O\prn*{\frac{1}{\sqrt{n}}}$ for a $(\alpha,\mu)$-\klshort{} loss will obtain an overall excess risk bound of order $O\prn*{\frac{\mu}{n^{\alpha/2}}}$. For $(1,\mu_1)$-\klshort{} this translates to an overall $O\prn*{\frac{\mu_1}{\sqrt{n}}}$ rate, whereas $(2,\mu_2)$-\klshort{} implies a $O\prn*{\frac{\mu_2}{n}}$ rate. The first rate becomes favorable when $\mu_1 \le \sqrt{n}\cdot\mu_2$ which typically happens for very high dimensional problems. For the examples we study, $\mu_1$ is related only to the radius of the set $\cW$, while $\mu_2$ depends inversely on the smallest eigenvalue of the population covariance and so is well-behaved only for low-dimensional problems unless one makes strong distributional assumptions.

An important feature of our analysis is that we need to establish the \klshort{} condition only for the population risk; for the examples we consider this is easy as long as we assume the model is well-specified. Once this is done, our convergence results hold for \emph{any} algorithm that works on the dataset $\crl*{(x_t,y_t)}_{t=1}^{n}$ and finds an approximate first-order stationary point with $\nrm*{\grad\emprisk(\walg)}\leq{}\veps$. First-order algorithms that find approximate stationary points assuming only smoothness of the loss have enjoyed a surge of recent interest \citep{ghadimi2013stochastic, reddi2016stochastic,allen2016variance, allenzhu2017natasha2}, so this is an appealing proposition.

\subsection{Vector Rademacher Complexities: The How}
The starting point for our uniform convergence bounds for gradients is to apply the standard tool of symmetrization---a vector-valued version, to be precise. To this end let us introduce a \emph{normed} variant of Rademacher complexity.
\begin{definition}[Normed Rademacher Complexity]
Given a vector valued class of function $\cF$ that maps the space $\cZ$ to a vector space equipped with norm $\nrm*{\cdot}$, we define the normed Rademacher complexity for $\cF$ on instances $\zr[n]$ via
\begin{equation}
\label{eq:rad_norm}
\rad_{\nrm*{\cdot}}(\cF\midsem{}\zr[n]) \ldef  \En_{\eps}\sup_{f\in\cF}\nrm*{\sum_{t=1}^{n}\eps_tf(z_t)}.
\end{equation} 
\end{definition}
With this definition we are ready to provide a straightforward generalization of the standard real-valued symmetrization lemma.

\begin{proposition}
 \label{prop:kl_rademacher}
For any $\delta > 0$, with probability at least $1 - \delta$ over the data $\crl*{(x_t,y_t)}_{t=1}^{n}$,
\begin{equation}
 \En\sup_{w \in \cW}\nrm*{\grad\emprisk(w) - \grad\poprisk(w)} \le \frac{4}{n} \cdot \rad_{\nrm*{\cdot}}(\grad{}\ls\circ\cW\midsem{}\xr[n],\yr[n]) +  c \left(\frac{\log\left(\frac{1}{\delta}\right)}{n}\right),
\end{equation}
where the constant $c$ depends only on the range of $\nrm*{\grad{}\ls}$.
\end{proposition}

To bound the complexity of the gradient class $\grad \ls \circ\cW$, we introduce a  chain rule for the normed Rademacher complexity that allows to easily control gradients of composition of functions.

\begin{theorem}[Chain Rule for Rademacher Complexity]
\label{thm:chain_rule}
Let sequences of functions $G_{t}:\bbR^{K}\to\bbR$ and $F_t:\bbR^{d}\to\bbR^{K}$ be given. Suppose there are constants $L_G$ and $L_F$ such that for all $1\leq{}t\leq{}n$,  $\nrm*{\grad{}G_t}_{2}\leq{}L_{G}$ and $\sqrt{\sum_{k=1}^{K}\nrm*{\grad{}F_{t,k}(w)}^{2}}\leq{}L_{F}$. Then,
\begin{equation}
\label{eq:chain_rule_vector}
\frac{1}{2}\En_{\eps}\sup_{w\in\cW}\nrm*{\sum_{t=1}^{n}\eps_{t}\grad{}(G_t(F_t(w)))}
\leq{} L_{F}\En_{\beps}\sup_{w\in\cW}\sum_{t=1}^{n}\tri*{\beps_{t},\grad{}G_t(F_t(w))}
 + L_{G}\En_{\beps}\sup_{w\in\cW}\nrm*{\sum_{t=1}^{n}\grad{}F_t(w)\beps_t},
\end{equation}
where $\grad{}F_t$ denotes the Jacobian of $F_t$, which lives in $\bbR^{d\times{}K}$, and $\beps\in\pmo^{K\times{}n}$ is a matrix of Rademacher random variables with $\beps_t$ denoting the $t$th column
\end{theorem}

The concrete learning settings we study---generalized linear models and robust regression---all involve composing non-convex losses and non-linearities or transfer functions with a linear predictor. That is, $\ell(w;x_t,y_t)$ can be written as $\ell(w;x_t,y_t) = G_t(F_t(w))$ where $G_t(a)$ is some $L$-Lipschitz function that possibly depends on $x_t$ and $y_t$  and $F_t(w) =  \tri*{w,x_t}$. In this case, the chain rule for derivatives gives us that $\nabla \ell(w;x_t,y_t) = G'_t(F_t(w)) \cdot  \nabla F_t(w) = G'_t( \tri*{w,x_t})) x_t$. Using the chain rule (with $K=1$), we conclude that 
\begin{align*}
\En_{\eps}\sup_{w\in\cW}\nrm*{\sum_{t=1}^{n}\eps_t\grad{}\ls(w\midsem{}x_t,y_t)}
&\leq  \mathbb{E}_{\epsilon}\left[\sup_{w \in \cW} \sum_{t=1}^n \epsilon_t G'_t( \tri*{w,x_t}))\right] + L\cdot\mathbb{E}_{\epsilon} \nrm*{\sum_{t=1}^{n}\epsilon_t x_t}.
\end{align*}
Thus, we have reduced the problem to controlling the Rademacher average for a real valued  function class of linear predictors and controlling the vector-valued random average $\mathbb{E}_{\epsilon} \nrm*{\sum_{t=1}^{n}\epsilon_t x_t}$. The first term is handled using classical Rademacher complexity tools. As for the second term, it is a standard result (\cite{Pisier75}; see \cite{kakade09complexity} for discussion in the context of learning theory) that for all smooth Banach spaces, and more generally Banach spaces of Rademacher type 2, one has
$\En_{\eps}\nrm*{\sum_{t=1}^{n}\eps_{t}x_t}=O(\sqrt{n})$; see \pref{app:preliminaries} for details.

The key tool used to prove \pref{thm:chain_rule}, which appears throughout the technical portions of this paper, is the  \emph{vector-valued Rademacher complexity} due to \cite{maurer2016vector}.
\begin{definition}
For a function class $\cG\subseteq{}\prn*{\cZ\to\bbR^{K}}$, the vector-valued Rademacher complexity is 
\begin{equation}
\radvec(g\midsem{}\zr[n]) \ldef \En_{\beps}\sup_{g\in\cG}\sum_{t=1}^{n}\tri*{\beps_t,g(z_t)}.
\end{equation}
\end{definition}

The vector-valued Rademacher complexity arises through an elegant contraction trick due to Maurer.
\begin{theorem}[Vector-valued contraction \citep{maurer2016vector}]
  \label{thm:contraction_vector}
  Let $\cG\subseteq{}(\cZ\to\bbR^{K})$, and let $h_{t}:\bbR^{K}\to\bbR$ be a sequence of functions for $t\in\brk*{n}$, each of which is $L$-Lipschitz with respect to $\ls_{2}$. Then
  \begin{equation}
  \En_{\eps}\sup_{g\in\cG}\sum_{t=1}^{n}\eps_th_t(g(z_t)) 
\leq{} \sqrt{2}L\cdot\radvec(\cG\midsem{}\zr[n]).
\end{equation}
\end{theorem}

We remark that while our applications require only gradient uniform convergence, we anticipate that the tools of this section will find use in settings where convergence of higher-order derivatives is needed to ensure success of optimization routines. To this end, we have extended the chain rule (\pref{thm:chain_rule}) to handle Hessian convergence; see \pref{app:structural}.


\section{Application: Smooth Models}
\label{sec:smooth}

In this section we instantiate the general gradient uniform convergence tools and the \klshort{} condition to derive optimization consequences for two standard settings previously studied by \cite{mei2016landscape}:  generalized linear models and robust regression.  

\paragraph{Generalized Linear Model}
\label{sec:glm}
We first consider the problem of learning a generalized linear model with the square loss.\footnote{\cite{mei2016landscape} refer to the model as ``binary classification'', since $\sigma$ can model conditional probabilities of two classes.}
Fix a norm $\nrm*{\cdot}$, take $\cX=\crl*{x\in\bbR^{d}\mid\nrm*{x}\leq{}R}$, $\cW=\crl*{w\in\bbR^d\mid{}\nrm*{w}_{\star}\leq{}B}$, and $\cY=\crl*{0,1}$. Choose a link function $\sigma{}:\bbR\to\brk*{0,1}$ and define the loss to be $\ls(w\midsem{}x,y) = \prn*{\sigma(\tri*{w,x})-y}^{2}$. Standard choices for $\sigma$ include the logistic link function $\sigma(s)=(1+e^{-s})^{-1}$ and the probit link function $\sigma(s) = \Phi(s)$, where $\Phi$ is the gaussian cumulative distribution function.

To establish the GD property and provide uniform convergence bounds, we make the following regularity assumptions on the loss.
\begin{assumption}[Generalized Linear Model Regularity]\label{ass:glm}~Let $\cS=\brk*{-BR,BR}$. \vspace{5pt}\\
\hspace*{1cm}(a) $\exists C_{\sigma}\geq{}1$ s.t. $\max\crl*{\sigma'(s), \sigma''(s)}\leq{}C_{\sigma}$ for all $s\in\cS$. \vspace{5pt}\\
\hspace*{1cm}(b) $\exists c_{\sigma}>0$ s.t. $\sigma'(s)\geq{}c_{\sigma}$ for all $s\in\cS$. \vspace{5pt}\\
\hspace*{1cm}(c) $\En\brk*{y\mid{}x}=\sigma\prn*{\tri*{w^{\star},x}}$ for some $w^{\star}\in\cW$.
\end{assumption}

Assumption (a) suffices to bound the normed Rademacher complexity $\radnorm(\grad{}\ls\circ\cW)$, and combined with (b) and (c) the assumption implies that $\poprisk$ satisfies three variants of \klshort{} condition, and this leads to three final rates: a dimension-independent ``slow rate'' that holds for any smooth norm, a dimension-dependent fast rate for the $\ls_2$ norm, and a sparsity-dependent fast rate that holds under an additional restricted eigenvalue assumption. This gives rise to a family of generic excess risk bounds. 

To be precise, let us introduce some additional notation: $\Sigma = \En_{x}\brk*{xx^{\trn}}$ is the data covariance matrix and  $\eigmin(\Sigma)$ denotes the minimum non-zero eigenvalue.  For sparsity dependent fast rates, define $\cC(S, \alpha)\ldef\crl*{\nu\in\bbR^{d}\mid{}\nrm*{\nu_{S^{C}}}_{1}\leq{}\alpha\nrm*{\nu_{S}}_{1}}$ and let $\eigminr(\Sigma)=\inf_{\nu\in\cC(S(w^{\star}),1)}\frac{\tri*{\nu,\Sigma{}\nu}}{\tri*{\nu,\nu}}$ be the restricted eigenvalue.\footnote{Recall that $S(w^{\star})\subseteq{}\brk*{d}$ is the set of non-zero entries of $w^{\star}$, and for any vector $w$, $w_{S}\in\bbR^{d}$ refers to the vector $w$ with all entries in $S^{C}$ set to zero (as in \citep{raskutti2010restricted}).}
Lastly, recall that a norm $\nrm{\cdot}$ is said to be $\beta$-smooth if the function $\Psi(x)=\frac{1}{2}\nrm*{x}^{2}$ has $\beta$-Lipschitz gradients with respect to $\nrm*{\cdot}$.

\begin{theorem}
\label{thm:glm_main} 
For the generalized linear model setting, the following excess risk inequalities each hold with probability at least $1 - \delta$ over the draw of the data $\crl*{(x_t,y_t)}_{t=1}^{n}$  for any algorithm $\walg$:\\
$\bullet$~ \textbf{Norm-Based/High-Dimensional Setup.}~~ When $\cX$ is the ball for $\beta$-smooth norm $\nrm*{\cdot}$ and $\cW$ is the dual ball,
\begin{equation*}
\poprisk(\walg) - \optrisk \leq  \mu_h\cdot\nrm*{\grad\emprisk(\walg)} + \frac{C_h}{\sqrt{n}}.
\end{equation*}
$\bullet$~ \textbf{Low-Dimensional $\ls_2/\ls_2$ Setup.}~~ When $\cX$ and $\cW$ are both $\ls_2$ balls:
\begin{equation*}
\poprisk(\walg) - \optrisk 
\leq  \frac{1}{\eigmin(\Sigma)}\prn*{\mu_l \cdot\nrm*{\grad\emprisk(\walg)}^{2} + \frac{C_l}{n}}.
\end{equation*}
$\bullet$~ \textbf{Sparse $\ls_{\infty}/\ls_1$ Setup.}~~  When $\cX$ is the $\ls_{\infty}$ ball, $\cW$ is the $\ls_1$ ball, and  $\nrm*{w^{\star}}_{1}=B$:\footnote{The constraint $\nrm*{w^{\star}}_{1}=B$ simplifies analysis of generic algorithms in the vein of the constrained LASSO \citep{tibshirani2015statistical}.} 
\begin{equation*}
\poprisk(\walg) - \optrisk 
\leq  \frac{\nrm*{w^{\star}}_{0}}{\eigminr(\Sigma)}\prn*{\mu_s\cdot\nrm*{\grad\emprisk(\walg)}^{2} + \frac{C_s}{n}}.
\end{equation*}
The quantities $C_h/C_l/C_s$ and $\mu_h/\mu_l/\mu_s$  are constants depending on $(B, R, C_\sigma, c_\sigma, \beta, \log(\delta^{-1}))$ but not explicitly on dimension (beyond logarithmic factors) or complexity of the class $\cW$ (beyond $B$ and $R$).
\end{theorem}
Precise statements for the problem dependent constants in \pref{thm:glm_main} including dependence on the norms $R$ and $B$ can be found in \pref{app:smooth}.

\begin{table}[ht]
\begin{center}
\begin{tabular}{|c | l | l | l | c |}
  \hline
   Model & \multicolumn{1}{c|}{Algorithm} & \multicolumn{2}{c|}{Sample Complexity} \\ \cline{3-4}
   ~ &  & Norm-based/Infinite dim. & Low-dim. \\ \hhline{*4-}
  Generalized Linear \cellcolor{blue!0} & \pref{prop:glm_opt} & $O(\veps^{-2})$ \cellcolor{blue!5}& $O(d\veps^{-1})$ \cellcolor{blue!5}\\ \hhline{~|*3-}
  \cellcolor{blue!0} & \cite{mei2016landscape} Theorem 4 & n/a & $O(d\veps^{-1})$ \cellcolor{blue!5}\\ \hhline{~|*3-}
   \cellcolor{blue!0} & \glmtron{} \citep{kakade2011efficient} &  $O(\veps^{-2})$ \cellcolor{blue!5}&  n/a\\ \hhline{~|*3-}
  \hline \hline
    Robust Regression \cellcolor{blue!0} & \pref{prop:glm_opt} & $O(\veps^{-2})$ \cellcolor{blue!5}& $O(d\veps^{-1})$ \cellcolor{blue!5}\\ \hhline{~|*3-}
  \cellcolor{blue!0} & \cite{mei2016landscape} Theorem 6 & n/a & $O(d\veps^{-1})$ \cellcolor{blue!5}\\ \hhline{~|*3-}
  \hline
\end{tabular}
\label{tab:complexity}
\caption{Sample complexity comparison. Highlighted cells indicate optimal sample complexity.} 
\vspace{0.2in}
\end{center}
\end{table}

We now formally introduce the robust regression setting and provide a similar guarantee.
\paragraph{Robust Regression} Fix a norm $\nrm*{\cdot}$ and take $\cX=\crl*{x\in\bbR^{d}\mid\nrm*{x}\leq{}R}$, $\cW=\crl*{w\in\bbR^d\mid{}\nrm*{w}_{\star}\leq{}B}$, and $\cY=\brk*{-Y,Y}$ for some constant $Y$. We pick a potentially non-convex function $\rho{}:\bbR\to\bbR_{+}$ and define the loss via $\ls(w\midsem{}x,y) = \rho\prn*{\tri*{w,x}-y}$. Non-convex choices for $\rho$ arise in robust statistics, with a canonical example being Tukey's biweight loss.\footnote{For a fixed parameter $c>0$ the biweight loss is defined via 
$\rho(t) = \frac{c^{2}}{6}\cdot\left\{\begin{array}{ll} 
1 - (1-(t/c)^2)^{3},&\quad\abs*{t}\leq{}c.\\
1,&\quad\abs*{t}\geq{}c.
\end{array}\right.$} While optimization is clearly not possible for arbitrary choices of $\rho$, the following assumption is sufficient to guarantee that the population risk $\poprisk$ satisfies the \klshort{} property.

\begin{assumption}[Robust Regression Regularity]\label{ass:rr}~Let $\cS=\brk*{-(BR+Y),(BR+Y)}$. \vspace{5pt}\\
\hspace*{1cm}(a) $\exists C_{\rho}\geq{}1$ s.t. $\max\crl*{\rho'(s),\rho''(s)}\leq{}C_{\rho}$ for all $s\in\cS$. \vspace{5pt}\\
\hspace*{1cm}(b) $\rho'$ is odd with $\rho'(s)>0$ for all $s>0$ and $h(s)\ldef{}\En_{\zeta}\brk*{\rho'(s+\zeta)}$ has $h'(0)>c_{\rho}$.\vspace{5pt}\\
\hspace*{1cm}(c) There is $w^{\star}\in\cW$ such that $y=\tri*{w^{\star},x}+\zeta$, and  $\zeta$ is symmetric zero-mean given $x$.
\end{assumption}
Similar to the generalized linear model setup, the robust regression setup satisfies three variants of the \klshort{} depending on assumptions on the norm $\nrm*{\cdot}$ and the data distribution.
\begin{theorem}
\label{thm:rr_main} For the robust regression setting, the following excess risk inequalities each hold with probability at least $1 - \delta$ over the draw of the data $\crl*{(x_t,y_t)}_{t=1}^{n}$ for any algorithm $\walg$:\\
$\bullet$~ \textbf{Norm-Based/High-Dimensional Setup.}~~When $\cX$ is the ball for $\beta$-smooth norm $\nrm*{\cdot}$ and $\cW$ is the dual ball,
\begin{equation*}
\poprisk(\walg) - \optrisk \leq  \mu_h\cdot\nrm*{\grad\emprisk(\walg)} + \frac{C_h}{\sqrt{n}}.
\end{equation*}
$\bullet$~ \textbf{Low-Dimensional $\ls_2/\ls_2$ Setup.}~~ When $\cX$ and $\cW$ are both $\ls_2$ balls:
\begin{equation*}
\poprisk(\walg) - \optrisk 
\leq  \frac{1}{\eigmin(\Sigma)}\prn*{\mu_l \cdot\nrm*{\grad\emprisk(\walg)}^{2} + \frac{C_l}{n}}.
\end{equation*}
$\bullet$~ \textbf{Sparse $\ls_{\infty}/\ls_1$ Setup.}~~  When $\cX$ is the $\ls_{\infty}$ ball, $\cW$ is the $\ls_1$ ball, and  $\nrm*{w^{\star}}_{1}=B$:
\begin{equation*}
\poprisk(\walg) - \optrisk 
\leq  \frac{\nrm*{w^{\star}}_{0}}{\eigminr(\Sigma)}\prn*{\mu_s\cdot\nrm*{\grad\emprisk(\walg)}^{2} + \frac{C_s}{n}}.
\end{equation*}
The constants $C_h/C_l/C_s$ and $\mu_h/\mu_l/\mu_s$  depend on $(B, R, C_\rho, c_\rho, \beta, \log(\delta^{-1}))$, but not explicitly on dimension (beyond $\log$ factors) or complexity of the class $\cW$ (beyond the range parameters $B$ and $R$). \end{theorem}

\pref{thm:glm_main} and \pref{thm:rr_main} immediately imply that standard non-convex \emph{optimization} algorithms for finding stationary points can be converted to non-convex \emph{learning} algorithms with optimal sample complexity; this is summarized by the following theorem, focusing on the ``high-dimensional'' and ``low-dimensional'' setups above in the case of the $\ls_2$ norm for simplicity.
 
\begin{proposition}
\label{prop:glm_opt}
Suppose that $\Sigma\succeq{}\frac{1}{d}I$, $\nrm*{\cdot}=\ls_2$. Consider the following meta-algorithm for the non-convex generalized linear model (under \pref{ass:glm}) and robust regression (under \pref{ass:rr}) setting.
\begin{enumerate}
\item Gather $n=\frac{1}{\veps^{2}}\wedge\frac{d}{\veps}$ samples $\crl*{(x_t,y_t)}_{t=1}^{n}$.
\item  Find a point $\walg\in\cW$ with $\nrm*{\grad{}\emprisk(\walg)}\leq{} O\prn*{\frac{1}{\sqrt{n}}}$, which is guaranteed to exist.
\end{enumerate}
This meta-algorithm guarantees $\En\poprisk(\walg) - \optrisk\leq{}C\cdot{}\veps$, where $C$ is a problem-dependent but dimension-independent\footnote{Whenever $B$ and $R$ are constant.} constant.
\end{proposition}
There are many non-convex optimization algorithms that provably find an approximate stationary point of the empirical risk, including gradient descent \citep{nesterov2013introductory}, SGD \citep{ghadimi2013stochastic}, and Non-convex SVRG \citep{reddi2016stochastic,allen2016variance}. Note, however, that these algorithms are not generically guaranteed to satisfy the constraint $\walg\in\cW$ a-priori. 
We can circumvent this difficulty and take advantage of these generic algorithms by instead finding stationary points of the \emph{regularized} empirical risk. We show that any algorithm that finds a (unconstrained) stationary point of the regularized empirical risk indeed the obtains optimal $O\prn*{\frac{1}{\veps^{2}}}$ sample complexity in the norm-based regime. 

\begin{manualtheorem}{\ref*{thm:regularized_glm}}[informal]Suppose we are in the generalized linear model setting or robust regression setting. Let $\emprisk^{\lambda}(w)=\emprisk(w) + \frac{\lambda}{2}\nrm*{w}_{2}^{2}$. For any $\delta>0$ there is a setting of the regularization parameter $\lambda$ such that any $\walg$ with $\grad{}\emprisk^{\lambda}(\walg)=0$
 guarantees $\poprisk(\walg) - \optrisk\leq{}\tilde{O}\prn*{
\sqrt{\frac{\log(\delta^{-1})}{n}}
}$ with probability at least $1-\delta$.
\end{manualtheorem}

See \pref{app:further_discussion} for the full theorem statement and proof.

Now is a good time to discuss connections to existing work in more detail.
\begin{enumerate}[label=\alph*)]
\item
The sample complexity $O\prn*{\frac{1}{\veps^{2}}\wedge{}\frac{d}{\veps}}$ for \pref{prop:glm_opt} is optimal up to dependence on Lipschitz constants and the range parameters $B$ and $R$ \citep{tsybakov2008introduction}.  The ``high-dimensional'' $O\prn*{\frac{1}{\veps^{2}}}$ regime is particularly interesting, and goes beyond recent analyses to non-convex statistical learning \citep{mei2016landscape}, which use arguments involving pointwise covers of the space $\cW$ and thus have unavoidable dimension-dependence. This highlights the power of the norm-based complexity analysis.

\item In the low-dimensional $O\prn*{\frac{d}{\veps}}$ sample complexity regime, \pref{thm:glm_main} and \pref{thm:rr_main} recovers the rates of \cite{mei2016landscape} under the same assumptions---see \pref{app:further_discussion} for details. Notably, this is the case even when the radius $R$ is not constant. Note however that when $B$ and $R$ are large the constants in \pref{thm:glm_main} and \pref{thm:rr_main} can be quite poor. For the logistic link it is only possible to guarantee $c_{\sigma}\geq{}e^{-BR}$, and so it may be more realistic to assume $BR$ is constant.

\item The \glmtron{} algorithm of \cite{kakade2011efficient} also obtains $O\prn*{\frac{1}{\veps^{2}}}$ for the GLM setting. Our analysis shows that this sample complexity does not require specialized algorithms; any first-order stationary point finding algorithm will do. \glmtron{} has no guarantees in the $O\prn*{\frac{d}{\veps}}$ regime, whereas our meta-algorithm works in both high- and low-dimensional regimes. A significant benefit of \glmtron, however, is that it does not require a lower bound on the derivative of the link function $\sigma$. It is not clear if this assumption can be removed from our analysis.

\item As an alternative approach, stochastic optimization methods for finding first-order stationary points can be used to directly find an approximate stationary point of the population risk $\nrm*{\poprisk(w)}\leq{}\veps$, so long as they draw a fresh sample at each step. In the high-dimensional regime it is possible to show that stochastic gradient descent (and for general smooth norms, mirror descent) obtains $O\prn*{\frac{1}{\veps^{2}}}$ sample complexity through this approach; this is sketched in \pref{app:further_discussion}. This approach relies on returning a randomly selected iterate from the sequence and only gives an in-expectation sample complexity guarantee, whereas \pref{thm:regularized_glm} gives a high-probability guarantee.

Also, note that many stochastic optimization methods can exploit the $(2,\mu)$-\klshort{} condition. Suppose we are in the low-dimensional regime with $\Sigma\succeq{}\frac{1}{d}I$. The fastest GD-based stochastic optimization method that we are aware of is SNVRG \citep{zhou2018stochastic}, which under the $(2,O(d))$-\klshort{} condition will obtain $\veps$ excess risk with $O\prn*{\frac{d}{\veps} + \frac{d^{3/2}}{\veps^{1/2}}}$ sample complexity.

\end{enumerate}
This discussion is summarized in \pref{tab:complexity}.


\section{Non-Smooth Models}
\label{sec:nonsmooth}
In the previous section we used gradient uniform convergence to derive immediate optimization and generalization consequences by finding approximate stationary points of smooth non-convex functions. In practice---notably in deep learning---it is common to optimize \emph{non-smooth} non-convex functions;  deep neural networks with rectified linear units (ReLUs) are the canonical example \citep{krizhevsky2012imagenet, he2016deep}. In theory, it is trivial to construct non-smooth functions for which finding approximate stationary points is intractable (see discussion in \cite{allenzhu2018how}), but it appears that in practice stochastic gradient descent can indeed find approximate stationary points of the empirical loss in standard neural network architectures \citep{zhang2016understanding}. It is desirable to understand whether gradient generalization can occur in this setting.

The first result of this section is a lower bound showing that even for the simplest possible non-smooth model---a single ReLU---it is impossible to achieve dimension-independent uniform convergence results similar to those of the previous section. On the positive side, we show that it \emph{is} possible to obtain dimension-independent rates under an additional margin assumption.

The full setting is as follows: $\cX\subseteq{}\crl*{x\in\bbR^{d}\mid\nrm*{x}_{2}\leq{}1}$, $\cW\subseteq{}\crl*{w\in\bbR^d\mid{}\nrm*{w}_{2}\leq{}1}$, $\cY=\crl*{-1,+1}$, and $\ls(w\midsem{}x,y) = \sigma(-\tri*{w,x}\cdot{}y)$, where $\sigma(s)=\max\crl*{s,0}$; this essentially matches the classical Perceptron setup. Note that the loss is not smooth, and so the gradient is not well-defined everywhere. Thus, to make the problem well-defined, we consider convergence for the following representative from the subgradient: $\grad\ls(w\midsem{}x,y)\ldef-y\ind\crl*{y\tri*{w,x}\leq{}0}\cdot{}x$.\footnote{For general non-convex and non-smooth functions one can extend this approach by considering convergence for a representative from the Clarke sub-differential \citep{borwein2010convex,clarke1990optimization}.} Our first theorem shows that gradient uniform convergence for this setup must depend on dimension, even when the weight norm $B$ and data norm $R$ are held constant.

\begin{theorem}
\label{thm:relu_lb_l2_l2}
Under the problem setting defined above, for all $n\in\bbN$ there exist a sequence of instances $\crl*{(x_t,y_t)}_{t=1}^{n}$ such that
\[
\En_{\eps}\sup_{w\in\cW}\nrm*{\sum_{t=1}^{n}\eps_{t}\grad{}\ls(w\midsem{}x_t,y_t)}_{2} = \Omega\prn*{\sqrt{dn}\wedge{}n}.
\]
\end{theorem}
This result contrasts the setting where $\sigma$ is smooth, where the techniques from \pref{sec:tools} easily yield a dimension-independent $O(\sqrt{n})$ upper bound on the Rademacher complexity. This is perhaps not surprising since the gradients are discrete functions of $w$, and indeed VC-style arguments suffice to establish the lower bound. 

In the classical statistical learning setting, the main route to overcoming dimension dependence---e.g., for linear classifiers---is to assume a \emph{margin}, which allows one to move from a discrete class to a real-valued class upon which a dimension-independent Rademacher complexity bound can be applied \citep{shalev2014understanding}. Such arguments have recently been used to derive dimension-independent function value uniform convergence bounds for deep ReLU networks as well \citep{bartlett2017spectrally,golowich2018size}. However, this analysis relies on one-sided control of the loss, so it is not clear whether it extends to the inherently directional problem of gradient convergence. Our main contribution in this section is to introduce additional machinery to prove dimension-free gradient convergence under a new type of margin assumption.

\begin{definition}
\label{def:soft_margin}
Given a distribution $P$ over the support $\cX$ and an increasing function $\phi:\brk*{0,1}\to\brk*{0,1}$, any $w \in \cW$ is said to satisfy the \textbf{\softmargin condition with respect to P} if
\begin{equation}
\forall \gamma \in [0, 1], ~~~\displaystyle\En_{x \sim P} \brk*{ \ind \crl*{  \tfrac{\abs{ \tri{w, x}}} {\nrm{w}_2\nrm{x}_2}  \leq \gamma }}  \leq \phi(\gamma).
\end{equation}
\end{definition}
We call $\phi$ a \emph{margin function}. We define the set of all weights that satisfy the \softmargin{} condition with respect to a distribution $P$ via:
\begin{equation}
\label{eq:w_phi_class}
\cW(\phi, P) = \crl*{w\in\cW\;:\;   \forall \gamma \in [0, 1],\ \ \En_{x \sim P} \brk*{  \ind\crl*{ \tfrac{\abs*{\tri*{w,x}}}{\nrm*{w}_2\nrm*{x}_2} \leq{} \gamma } } \leq \phi(\gamma)}.
\end{equation}

Of particular interest is $\cW(\phi, \empdist)$, the set of all the weights that satisfy the \softmargin{} condition with respect to the empirical data distribution. That is, any $w \in \cW(\phi, \empdist)$  predicts with at least a $\gamma$ margin on all but a $\phi(\gamma)$ fraction of the data. 
The following theorem provides a dimension-independent uniform convergence bound for the gradients over the class $\cW(\phi,\empdist)$ for any margin function $\phi$ fixed in advance.

\begin{theorem}
\label{thm:relu_margin}
Let $\phi:\brk*{0,1}\to\brk*{0,1}$ be a fixed margin function.
With probability at least $1-\delta$ over the draw of the data $\crl*{(x_t,y_t)}_{t=1}^{n}$,
\[
\sup_{w\in\cW(\phi,\empdist)}\nrm*{\grad{}\poprisk(w)-\grad{}\emprisk(w)}_{2}\leq{}\tilde{O}\left(\inf_{\gamma>0}\crl*{\sqrt{\phi(4\gamma)} + \frac{1}{\gamma}\sqrt{\frac{\log\prn*{\frac{1}{\delta}}}{n}} +  \frac{1}{\gamma^{\frac{1}{2}} n^{\frac{1}{4}}}}\right),
\]
where $\tilde{O}(\cdot)$ hides $\log \log(\frac{1}{\gamma})$ and $\log{}n$ factors.
\end{theorem}
As a concrete example, when $\phi(\gamma)=\gamma^{\frac{1}{2}}$ \pref{thm:relu_margin} yields a dimension-independent uniform convergence bound of $O(n^{-\frac{1}{12}})$, thus circumventing the lower bound of \pref{thm:relu_lb_l2_l2} for large values of $d$.


\section{Discussion}

We showed that vector Rademacher complexities are a simple and effective tool for deriving dimension-independent uniform convergence bounds and used these bounds in conjunction with the (population) \kl{} property to derive optimal algorithms for non-convex statistical learning in high and infinite dimension. We hope that these tools will find broader use for norm-based capacity control in non-convex learning settings beyond those considered here. Of particular interest are models where convergence of higher-order derivatives is needed to ensure success of optimization routines. \pref{app:structural} contains an extension of \pref{thm:chain_rule} for Hessian uniform convergence, which we anticipate will find use in such settings.

In \pref{sec:smooth} we analyzed generalized linear models and robust regression using both the $(1,\mu)$-\klshort property and the $(2,\mu)$-\klshort property. In particular, the $(1,\mu)$-\klshort property was critical to obtain dimension-independent norm-based capacity control. While there are many examples of models for which the population risk satisfies $(2,\mu)$-\klshort{}  property (phase retrieval \citep{sun2016geometric,zhang2016median}, ResNets with linear activations \citep{hardt2016identity}, matrix factorization \citep{liu2016quadratic}, blind deconvolution \citep{li2016rapid}), we do not know whether the $(1,\mu)$-\klshort property holds for these models. Establishing this property and consequently deriving dimension-independent optimization guarantees is an exciting future direction.

Lastly, an important question is to analyze non-smooth problems beyond the simple ReLU example considered in \pref{sec:nonsmooth}. See \cite{davis2018uniform} for subsequent work in this direction.

\paragraph{Acknowledgements} K.S acknowledges support from the NSF under grants CDS\&E-MSS 1521544 and NSF CAREER Award 1750575, and the support of an Alfred P. Sloan Fellowship. D.F. acknowledges support from the NDSEG PhD fellowship and Facebook PhD fellowship.

\bibliography{refs}
\appendix

\section{Preliminaries}
\label{app:preliminaries}

\subsection{Contraction Lemmas}

\begin{lemma}[e.g. \cite{LedouxTalagrand91}]
  \label{lem:scalar_contraction}
  Let $\cF$ be any scalar-valued function class and $\phi_{1},\ldots,\phi_{n}$ be any sequence of functions where $\phi_{t}:\bbR\to\bbR$ is $L$-Lipschitz. Then
  \begin{equation}
    \En_{\eps}\sup_{f\in\cF}\sum_{t=1}^{n}\eps_{t}\phi_{t}(f(x_t)) \leq{} L\cdot\En_{\eps}\sup_{f\in\cF}\sum_{t=1}^{n}\eps_{t}f(x_t).
  \end{equation}
\end{lemma}

The following is a weighted generalization of the vector-valued Lipschitz contraction inequality.
\begin{lemma}
  \label{lem:contraction_vector_weighted}
  Let $\cF\subseteq{}(\cX\to\bbR^{K})$, and let $h_{t}:\bbR^{K}\to\bbR$ be a sequence of functions for $t\in\brk*{n}$. Suppose each $h_{t}$ is $1$-Lipschitz with respect to $\nrm*{z}_{A_{t}}\ldef\sqrt{\tri*{z,A_{t}z}}$, where $A_{t}\in\bbR^{K\times{}K}$ is positive semidefinite. Then
  \begin{equation}
  \En_{\eps}\sup_{f\in\cF}\sum_{t=1}^{n}\eps_th_t(f(x_t)) 
\leq{} \sqrt{2}\En_{\mb{\eps}}\sup_{f\in\cF}\sum_{t=1}^{n}\tri*{\mb{\eps}_{t}, A^{1/2}_{t}f(x_t)}.
\end{equation}
\end{lemma}
\begin{proof}[Proof sketch for \pref{lem:contraction_vector_weighted}]
Same proof as Theorem 3 in \cite{maurer2016vector}, with the additional observation that $\nrm*{z}_{A_t}=\nrm*{A_t^{1/2}z}_{2}$.
\end{proof}

\begin{lemma}
  \label{lem:block_contraction}
  Let $\cG$ be a class of vector-valued functions whose output space forms $M$ blocks of vectors, i.e. each $g\in\cG$ has the form $g:\cZ\to\bbR^{d_1+d_2+\cdots+d_{M}}$, where $g(z)_{i}\in\bbR^{d_i}$ denotes the $i$th block. Let $h_{t}:\bbR^{d_1+d_2+\cdots+d_{M}}\to\bbR$, be a sequence of functions for $t\in\brk*{n}$ that satisfy the following block-wise Lipschitz property: For any assignment $a_{1},\ldots,a_{M}$ with each $a_{i}\in\bbR^{d_i}$, $h_{t}(a_1,\ldots,a_{M})$ is $L_{i}$-Lipschitz with respect to $a_i$ in the $\ls_{2}$ norm. Then
  \[
    \En_{\eps}\sup_{g\in\cG}\sum_{t=1}^{n}h_{t}(g_{1}(z_t),\ldots,g_{M}(z_t)) \leq{} \sqrt{2M}\sum_{i=1}^{M}L_{i}\En_{\mb{\eps}}\sup_{f\in\cF}\sum_{t=1}^{n}\tri*{\mb{\eps}_{t},g_{i}(z_t)}.
  \]
\end{lemma}
\begin{proof}
Immediate consequence of \pref{lem:contraction_vector_weighted}, along with sub-additivity of the supremum.
\end{proof}

\subsection{Bound for Vector-Valued Random Variables}

\begin{definition}
\label{def:smooth}
  For any vector space $\cV$, a convex function $\Psi:\cV\to\bbR$ is $\beta$-smooth with respect to a norm $\nrm*{\cdot}$ if
  \[
    \Psi(x) \leq{} \Psi(y) + \tri*{\grad\Psi(y), x-y} + \frac{\beta}{2}\nrm*{x-y}^{2}\quad\forall{}x,y\in\cV.
  \]
A norm $\nrm{\cdot}$ is said to be $\beta$-smooth if the function $\Psi(x)=\frac{1}{2}\nrm*{x}^{2}$ is $\beta$-smooth with respect to $\nrm*{\cdot}$.
\end{definition}

\begin{theorem}
\label{thm:smooth_type}
Let $\nrm*{\cdot}$ be any norm for which there exists $\Psi$ such that $\Psi(x)\geq\frac{1}{2}\nrm*{x}^{2}$, $\Psi(0)=0$, and $\Psi$ is $\beta$-smooth with respect to $\nrm*{\cdot}$. Then
  \[
  \En_{\eps}\nrm*{\sum_{t=1}^{n}\eps_{t}x_{t}} \leq{}\sqrt{\beta\sum_{t=1}^{n}\nrm*{x_t}^{2}}.
  \]
\end{theorem}
The reader may consult \cite{Pinelis94} for a high-probability version of this theorem.
    \begin{fact}
      \label{fact:smooth}
The following spaces and norms satisfy the preconditions of \pref{thm:smooth_type}: 
\begin{itemize}
\item $(\bbR^{d},\ls_{p})$ for $p\geq{}2$, with $\beta=p-1$ \citep{kakade09complexity}.
\item $(\bbR^{d},\ls_{\infty})$, with $\beta=O(\log{}d)$ \citep{kakade09complexity}.
\item $(\bbR^{d_1\times{}d_2},\nrm*{\cdot}_{\sigma})$, with $\beta=O(\log(d_1 + d_2))$ \citep{kakade2012regularization}.
\end{itemize}
\end{fact}
\begin{proof}[\pfref{thm:smooth_type}]
Using Jensen's inequality and the upper bound property of $\Psi$ we have
\begin{align*}
\En_{\eps}\nrm*{\sum_{t=1}^{n}\eps_{t}x_{t}} &\leq{}\sqrt{\En_{\eps}\nrm*{\sum_{t=1}^{n}\eps_{t}x_{t}}^{2}}
\leq{}\sqrt{2}\cdot{}\sqrt{\En_{\eps}\Psi\prn*{\sum_{t=1}^{n}\eps_{t}x_{t}}} .
\end{align*}
Applying the smoothness property at time $n$, and using that $\eps_n$ is independent of $\eps_1,\ldots,\eps_{n-1}$:
\begin{align*}
\sqrt{\En_{\eps}\Psi\prn*{\sum_{t=1}^{n}\eps_{t}x_{t}}} 
&\leq{}\sqrt{\En_{\eps}\brk*{\Psi\prn*{\sum_{t=1}^{n-1}\eps_{t}x_{t}}
+ \tri*{\Psi\prn*{\sum_{t=1}^{n-1}\eps_{t}x_{t}}, \eps_n{}x_n}
+ \frac{\beta}{2}\nrm*{x_n}^{2}
}}
&=\sqrt{\En_{\eps}\Psi\prn*{\sum_{t=1}^{n-1}\eps_{t}x_{t}}
+ \frac{\beta}{2}\nrm*{x_n}^{2}
}.
\end{align*}
The result follows by repeating this argument from time $t=n-1$ to $t=1$.
\end{proof}

\section{Proofs from \pref{sec:tools}}
\label{app:tools}

\begin{theorem}[\cite{bartlett2005local}, Theorem A.2/Lemma A.5]
\label{thm:uniform_convergence_general}
Let $\cF\subseteq{}(\cZ\to\bbR)$ be a class of functions. Let $Z_{1},\ldots,Z_{n}\sim{}\cD$ i.i.d. for some distribution $\cD$. Then with probability at least $1-\delta$ over the draw of $Z_{1:n}$,
\begin{equation}
\label{eq:uniform_convergence_general}
\En\sup_{f\in\cF}\abs*{\En_{Z}f(Z)-\frac{1}{n}\sum_{t=1}^{n}f(Z_t)} \leq{} 4\En_{\eps}\sup_{f\in\cF}\frac{1}{n}\sum_{t=1}^{n}\eps_tf(Z_t) + 4\sup_{f\in\cF}\sup_{z\in\cZ}\abs*{f(Z)}\cdot\frac{\log\prn*{\frac{2}{\delta}}}{n}.
\end{equation}
\end{theorem}

\begin{lemma}[Uniform convergence for vector-valued functions]
  \label{lem:uniform_convergence}
Let $\cG\subseteq{}\crl*{g:\cZ\to\Bspace}$ for arbitrary set $\cZ$ and vector space $\Bspace$. Let $Z_{1},\ldots,Z_{n}\sim{}\cD$ i.i.d. for some distribution $\cD$. Let a norm $\nrm*{\cdot}$ over $\Bspace$ be fixed. Then with probability at least $1-\delta$ over the draw of $Z_{1:n}$,
\begin{equation}
  \label{eq:uniform_convergence}
  \En\sup_{g\in\cG}\nrm*{\En_{Z}g(Z)-\frac{1}{n}\sum_{t=1}^{n}g(Z_t)} \leq{} 4\En_{\eps}\sup_{g\in\cG}\nrm*{\frac{1}{n}\sum_{t=1}^{n}\eps_tg(Z_t)} + 4\sup_{g\in\cG}\sup_{Z\in\cZ}\nrm*{g(Z)}\cdot\frac{\log\prn*{\frac{2}{\delta}}}{n}
\end{equation}
for some absolute constant $c>0$.
\end{lemma}
\begin{proof}[\pfref{lem:uniform_convergence}]
This follows immediately by applying \pref{thm:uniform_convergence_general} to the expanded function class $\cF\ldef\crl*{Z\mapsto{}\tri*{g(Z),v}\mid{}g\in\cG, \nrm*{v}_{\star}\leq{}1}$.
\end{proof}

\begin{proof}[\pfref{prop:kl_unif}]
This is a direct consequence of McDiarmid's inequality. Consider any vector-valued function class of functions $\cG$. Let $Z_{1},\ldots,Z_{n}\sim{}\cD$ i.i.d. for some distribution $\cD$. Then McDiarmid's inequality implies that with probability at least $1-\delta$ over the draw of $Z_{1:n}$, 
\begin{equation}
\label{eq:mcdiarmid}
  \sup_{g\in\cG}\nrm*{\En_{Z}g(Z)-\frac{1}{n}\sum_{t=1}^{n}g(Z_t)} \leq{} \En\sup_{g\in\cG}\nrm*{\En_{Z}g(Z)-\frac{1}{n}\sum_{t=1}^{n}g(Z_t)} + c\cdot\sup_{g\in\cG}\sup_{Z\in\cZ}\nrm*{g(Z)}\cdot\sqrt{\frac{\log\prn*{\frac{2}{\delta}}}{n}}.
\end{equation}
\end{proof}

\begin{proof}[\pfref{prop:kl_rademacher}]
This follows by applying the uniform convergence lemma, \pref{lem:uniform_convergence}, to the class $\cG = \crl*{(x,y) \mapsto{}\grad{}\ls(w\midsem{}x,y)\mid{}w\in\cW}$.
\end{proof}

\begin{proof}[\pfref{thm:chain_rule}]
We write
\[
\En_{\eps}\sup_{w\in\cW}\nrm*{\sum_{t=1}^{n}\eps_{t}\grad{}(G_t(F_t(w)))}
= \En_{\eps}\sup_{w\in\cW}\sup_{v\in\Bspace^{\star}:\nrm*{v}_{\star}\leq{}1}
\sum_{t=1}^{n}\eps_{t}\tri*{\grad{}(G_t(F_t(w))),v},
\]
Using the chain rule for differentiation we have

\[
\tri*{\grad{}(G_t(F_t(w))),v} = \tri*{(\grad{}G_t)(F_t(w)), (\tri*{\grad{}F_{t,k}(w), v})_{k\in\brk*{K}}}.
\]
We now introduce new functions that relabel the quantities in this expression. Let $h:\bbR^{2K}\to\bbR$ be given by $h(a,b) = \tri*{a,b}$, let $f_1:\cW\to\bbR^{K}$ be given by $f_{1}(w) = (\grad{}G_t)(F_t(w))$ and $f_{2}$ be given by $f_2(w,v) = (\tri*{\grad{}F_{t,k}(w), v})_{k\in\brk*{K}}$. We apply the block-wise contraction lemma \pref{lem:block_contraction} with one block for $f_1$ and one block for $f_2$ to conclude
\begin{align*}
&\En_{\eps}\sup_{w\in\cW}\sup_{v\in\Bspace^{\star}:\nrm*{v}_{\star}\leq{}1}
\sum_{t=1}^{n}\eps_{t}h(f_1(w),f_2(w,v)) \\
&\leq{} 2L_{F}\En_{\eps}\sup_{w\in\cW}\sup_{v\in\Bspace^{\star}:\nrm*{v}_{\star}\leq{}1}
\sum_{t=1}^{n}\tri*{\beps_t,f_1(w)}
+ 2L_{G}\En_{\eps}\sup_{w\in\cW}\sup_{v\in\Bspace^{\star}:\nrm*{v}_{\star}\leq{}1}
\sum_{t=1}^{n}\tri*{\beps_t,f_2(w,v)},
\end{align*}
which establishes the result after expanding terms. All that must be verified is that the assumptions on the norm bounds for $\grad{}G_t$ and $\grad{}F_t$ in the theorem statement ensure the the Lipschitz requirement in the statement of \pref{lem:block_contraction} is met.
\end{proof}


\section{Proofs from \pref{sec:smooth}}
\label{app:smooth}

For all proofs in this section we adopt the notation $s\ldef\nrm*{w^{\star}}_{0}$, and use $c>0$ to denote an absolute constant whose precise value depends on context.

\subsection{Generalized Linear Models}
\begin{proof}[\pfref{thm:glm_main}]
To begin, we apply \pref{prop:kl_unif} and \pref{prop:kl_rademacher} to conclude that whenever $(\alpha,\mu)$-PL holds, with probability at least $1-\delta$ over the examples $\crl*{(x_t,y_t)}_{t=1}^{n}$, any learning algorithm $\walg\in\cW$ satisfies
\begin{equation}
\label{eq:glm_uniform}
\poprisk(\walg) - \optrisk \leq{} c\cdot{}\mu\prn*{\nrm*{\grad\emprisk(\walg)}^{\alpha} + \prn*{\frac{\rad_{\nrm*{\cdot}}(\grad{}\ls\circ\cW\midsem{}\xr[n],\yr[n])}{n} + 2C_{\sigma}R\sqrt{\frac{\log(1/\delta)}{n}}}^{\alpha}}.
\end{equation}
Here $c>0$ is an absolute constant and we have used that $\nrm*{\grad{}\ls(w\midsem{}x_t,y_t)}\leq{}2C_{\sigma}R$.

\paragraph{Smooth high-dimensional setup} For the general smooth norm pair setup in \pref{eq:glm_uniform}, \pref{lem:glm_kl} and \pref{lem:glm_gradient} imply

\begin{align*}
\poprisk(\walg) - \optrisk &\leq{} c\cdot{}\frac{BC_{\sigma}}{c_{\sigma}}\prn*{\nrm*{\grad\emprisk(\walg)} + \prn*{BR^{2}C_{\sigma}^{2}\sqrt{\frac{\beta{}}{n}} + 2C_{\sigma}R\sqrt{\frac{\log(1/\delta)}{n}}}} \\
&=  \muhigh\cdot\nrm*{\grad\emprisk(\walg)} + \frac{\Chigh}{\sqrt{n}}.
\end{align*}
where we recall $\Chigh=c\cdot\frac{B^{2}R^{2}C_{\sigma}^{3}\sqrt{\beta}+2C_{\sigma}^{2}BR\sqrt{\log(1/\delta)}}{c_{\sigma}}$ and $\muhigh=c\cdot{}\frac{BC_{\sigma}}{c_{\sigma}}$.

\paragraph{Low-dimensional $\ls_2/\ls_2$ setup} For the low-dimension $\ls_2/\ls_2$ pair setup in \pref{eq:glm_uniform},  \pref{lem:glm_kl} and \pref{lem:glm_gradient} imply

\begin{align*}
\poprisk(\walg) - \optrisk &\leq{} c\cdot{}\frac{C_{\sigma}}{4c_{\sigma}^{3}\eigmin\prn*{\Sigma}}\prn*{\nrm*{\grad\emprisk(\walg)}^{2} + \prn*{BR^{2}C_{\sigma}^{2}\sqrt{\frac{1}{n}} + 2C_{\sigma}R\sqrt{\frac{\log(1/\delta)}{n}}}^2} \\
&=  \frac{\mulow}{\eigmin(\Sigma)}\cdot\nrm*{\grad\emprisk(\walg)}^{2} + \frac{\Clow}{n\cdot\eigmin(\Sigma)},
\end{align*}
where we have used that the $\ls_2$ norm is $1$-smooth in \pref{lem:glm_gradient}. Recall that
$\Clow = c\cdot{}\frac{2C_{\sigma}^{5}R^{4}B^{2} + 8C_{\sigma}^{3}R^{2}\log(1/\delta)}{4c_{\sigma}^{3}}$ and $\mulow=c\cdot{}\frac{C_{\sigma}}{4c_{\sigma}^{3}}$.

\paragraph{Sparse $\ls_{\infty}/\ls_1$ setup} For the sparse $\ls_{\infty}/\ls_1$ pair setup in \pref{eq:glm_uniform},  \pref{lem:glm_kl} and \pref{lem:glm_gradient} imply
\begin{align*}
\poprisk(\walg) - \optrisk &\leq{} c\cdot{}\frac{C_{\sigma}s}{c_{\sigma}^{3}\eigminr\prn*{\Sigma}}\prn*{\nrm*{\grad\emprisk(\walg)}^{2} + \prn*{BR^{2}C_{\sigma}^{2}\sqrt{\frac{\log{}d}{n}} + 2C_{\sigma}R\sqrt{\frac{\log(1/\delta)}{n}}}^2} \\
&=  \frac{\musparse\cdot{}s}{\eigminr(\Sigma)}\cdot\nrm*{\grad\emprisk(\walg)}^{2} + \frac{s}{n}\cdot\frac{\Csparse}{\eigminr(\Sigma)},
\end{align*}
where we have used that the $\ls_{\infty}$ norm has the smoothness property with $\beta=O(\log(d))$ in \pref{lem:glm_gradient}. Recall that
$\Csparse = c\cdot{}\frac{2C_{\sigma}^{5}R^{4}B^{2}\log(d) + 8C_{\sigma}^{3}R^{2}\log(1/\delta)}{c_{\sigma}^{3}}$ and $\musparse=c\cdot{}\frac{C_{\sigma}}{c_{\sigma}^{3}}$.

\end{proof}
\begin{lemma}[\klshort{} condition for the GLM]
\label{lem:glm_kl}
Consider the generalized linear model setup of \pref{sec:glm}.
\begin{itemize}[leftmargin=*]
\item When $\nrm*{\cdot}/\nrm*{\cdot}_{\star}$ are any dual norm pair, we have $\prn*{1, \frac{BC_{\sigma}}{c_{\sigma}}}$-\klshort{}:
\begin{equation}
\poprisk(w) - \poprisk(w^{\star}) \leq{} \frac{BC_{\sigma}}{c_{\sigma}}\nrm*{\grad{}\poprisk(w)}\quad\forall{}w\in\cW.
\end{equation}
\item In the $\ls_{2}/\ls_{2}$ setup, we have $\prn*{2,\frac{C_{\sigma}}{4c_{\sigma}^{3}\eigmin\prn*{\Sigma}}}$-\klshort{}:
\begin{equation}
\poprisk(w) - \poprisk(w^{\star}) \leq{} \frac{C_{\sigma}}{4c_{\sigma}^{3}\eigmin\prn*{\Sigma}}\nrm*{\grad{}\poprisk(w)}_{2}^{2}\quad\forall{}w\in\cW.
\end{equation}
\item In the sparse $\ls_{\infty}/\ls_{1}$ setup, where $\nrm*{w^{\star}}_{0}\leq{}s$, we have $\prn*{2,\frac{C_{\sigma}s}{c_{\sigma}^{3}\eigminr\prn*{\Sigma}}}$-\klshort{}:
\begin{equation}
\poprisk(w) - \poprisk(w^{\star}) \leq{} \frac{C_{\sigma}s}{c_{\sigma}^{3}\eigminr\prn*{\Sigma}}\nrm*{\grad{}\poprisk(w)}_{\infty}^{2}\quad\forall{}w\in\cW.
\end{equation}

\end{itemize}

\end{lemma}
\begin{proof}[\pfref{lem:glm_kl}]~\\
\textbf{Upper bound for excess risk.} We first prove the following intermediate upper bound:
\begin{equation}
\label{eq:glm_kl_main}
\poprisk(w) - \poprisk(w^{\star}) \leq{} \frac{C_{\sigma}}{2c_{\sigma}}\tri*{\grad{}\poprisk(w),w-w^{\star}}.
\end{equation}

Letting $w\in\cW$ be fixed, we have
\begin{align*}
\tri*{\grad{}\poprisk(w),w-w^{\star}} &= 2\En_{(x,y)}\brk*{
(\sigma(\tri*{w,x}-y)\sigma'(\tri*{w,x})\tri*{w-w^{\star},x}
}. 
\intertext{Using the well-specified assumption:}
&= 2\En_{x}\brk*{
(\sigma(\tri*{w,x}-\sigma(\tri*{w^{\star},x}))\sigma'(\tri*{w,x})\tri*{w-w^{\star},x}
}.
\end{align*}
We now consider the term inside the expectation. Since $\sigma$ is increasing we have
\[
\sigma(\tri*{w,x}-\sigma(\tri*{w^{\star},x}))\sigma'(\tri*{w,x})\tri*{w-w^{\star},x}
= \abs*{\sigma(\tri*{w,x}-\sigma(\tri*{w^{\star},x}))}\cdot\abs*{\tri*{w-w^{\star},x}}\cdot\sigma'(\tri*{w,x})
\]
point-wise. We apply two lower bounds. First, $\sigma'(\tri*{w,x})>c_{\sigma}$ by assumption. Second, Lipschitzness of $\sigma$ implies
\[
\abs*{\sigma(\tri*{w,x}) - \sigma(\tri*{w^{\star},x})}\leq{} C_{\sigma}\abs*{\tri*{w-w^{\star},x}}.
\]
Combining these inequalities, we also obtain the following inequality in expectation over $x$:
\begin{align*}
\En_{x}\prn*{\sigma(\tri*{w,x})-\sigma(\tri*{w^{\star},x})}^{2} &\leq{} \frac{C_{\sigma}}{2c_{\sigma}}\tri*{\grad{}\poprisk(w),w-w^{\star}}.
\end{align*}
Lastly, since the model is well-specified we have
\[
\poprisk(w) - \poprisk(w^{\star}) = \En_{x}\prn*{\sigma(\tri*{w,x}) - \sigma(\tri*{w^{\star},x})}^{2},
\]
by a standard argument:
\begin{align*}
\poprisk(w) - \poprisk(w^{\star})
&= \En_{x,y}\brk*{\sigma^{2}(\tri*{w,x}) + y^{2} - 2\sigma(\tri*{w,x})y - \sigma^{2}(\tri*{w^{\star},x}) - y^{2} + 2\sigma(\tri*{w^{\star},x})y
} \\
&= \En_{x}\brk*{\sigma^{2}(\tri*{w,x}) - 2\sigma(\tri*{w,x})\sigma(\tri*{w^{\star},x})  +\sigma^{2}(\tri*{w^{\star},x})} \\
&= \En_{x}\prn*{\sigma(\tri*{w,x}) - \sigma(\tri*{w^{\star},x})}^{2}.
\end{align*}

\textbf{Proving the \klshort{} conditions.}
With the inequality \pref{eq:glm_kl_main} established the various \klshort{} inequalities follow in quick succession.
\begin{itemize}[leftmargin=*]
\item $\prn*{1, \frac{BC_{\sigma}}{c_{\sigma}}}$-\klshort{}: \\~\\
To prove this inequality, simply user \Holder{}'s inequality to obtain the upper bound,
\[
\tri*{\grad{}\poprisk(w),w-w^{\star}} \leq{} 2B\nrm*{\grad{}\poprisk(w)}.
\]
\item $\prn*{2,\frac{C_{\sigma}}{4c_{\sigma}^{3}\eigmin\prn*{\En\brk*{xx^{\trn}}}}}$-\klshort{}:\\~\\
Resuming from \pref{eq:glm_kl_main} we have
\begin{align*}
\poprisk(w) - \poprisk(w^{\star}) &\leq{} \frac{C_{\sigma}}{2c_{\sigma}}\tri*{\grad{}\poprisk(w),w-w^{\star}}.
\intertext{Let $P_{\cX}$ denote the orthogonal projection onto $\mathrm{span}(\En\brk*{xx^{\trn}})$. Note that $\grad{}\ls(w\midsem{}x,y)$ is parallel to $x$, we can thus introduce the projection matrix $P_{\cX}$ while preserving the inner product}
&= \frac{C_{\sigma}}{2c_{\sigma}}\tri*{P_{\cX}\grad{}\poprisk(w),P_{\cX}\prn*{w-w^{\star}}}
\intertext{Applying Cauchy-Schwarz:}
&\leq{} \frac{C_{\sigma}}{2c_{\sigma}}\nrm*{\grad{}\poprisk(w)}_{2}\cdot\nrm*{P_{\cX}\prn*{w-w^{\star}}}_{2}. \numberthis\label{eq:glm_norm}
\end{align*}
What remains is to relate the gradient norm to the term $\nrm*{P_{\cX}\prn*{w-w^{\star}}}_{2}$. We proceed with another lower bound argument similar to the one used to establish \pref{eq:glm_kl_main}, 
\begin{align*}
\tri*{\grad{}\poprisk(w),w-w^{\star}} &= 2\En_{(x,y)}\brk*{
(\sigma(\tri*{w,x}-y)\sigma'(\tri*{w,x})\tri*{w-w^{\star},x}
}. 
\intertext{Using the well-specified assumption once more:}
&= 2\En_{x}\brk*{
(\sigma(\tri*{w,x}-\sigma(\tri*{w^{\star},x}))\sigma'(\tri*{w,x})\tri*{w-w^{\star},x}
}.
\intertext{Monotonicity of $\sigma$, implies the argument to the expectation is non-negative pointwise, so we have the lower bound,}
&\geq{} 2c_{\sigma}\En_{x}\brk*{
(\sigma(\tri*{w,x}-\sigma(\tri*{w^{\star},x}))\tri*{w-w^{\star},x}
}.
\end{align*}
Consider a particular draw of $x$ and assume $\tri*{w,x}\geq{}\tri*{w^{\star},x}$ without loss of generality. Using the mean value theorem, there is some $s\in\brk*{\tri*{w^{\star},x}, \tri*{w,x}}$ such that 
\[
(\sigma(\tri*{w,x}-\sigma(\tri*{w^{\star},x}))\tri*{w-w^{\star},x} = \tri*{w-w^{\star},x}^{2}\sigma'(s) \geq{} = \tri*{w-w^{\star},x}^{2}c_{\sigma}.
\]
Grouping terms, we have shown
\begin{align*}
\tri*{P_{\cX}\grad{}\poprisk(w),P_{\cX}\prn*{w-w^{\star}}} = \tri*{\grad{}\poprisk(w),w-w^{\star}} &\geq{} 2c_{\sigma}^{2}\En\tri*{w-w^{\star},x}^{2} \\
&= 2c_{\sigma}^{2}\tri*{w-w^{\star},\En\brk*{xx^{\trn}}(w-w^{\star})} \numberthis\label{eq:glm_quadratic}\\
&\geq{} 2c_{\sigma}^{2}\eigmin\prn*{\En\brk*{xx^{\trn}}}\nrm*{P_{\cX}(w-w^{\star})}_{2}^{2}.
\end{align*}
In other words, by rearranging and applying Cauchy-Schwarz we have
\[
\nrm*{P_{\cX}(w-w^{\star})}_{2} \leq{} \frac{1}{2c_{\sigma}^{2}\eigmin\prn*{\En\brk*{xx^{\trn}}}}\cdot\nrm*{\grad{}\cL_{\cD}(w)}_{2}.
\]
Combining this inequality with \pref{eq:glm_norm}, we have
\[
\poprisk(w) - \poprisk(w^{\star}) \leq{} \frac{C_{\sigma}}{4c_{\sigma}^{3}\eigmin\prn*{\En\brk*{xx^{\trn}}}}\cdot\nrm*{\grad{}\cL_{\cD}(w)}_{2}^{2}.
\]
\item $\prn*{2,\frac{C_{\sigma}s}{c_{\sigma}^{3}\eigminr\prn*{\En\brk*{xx^{\trn}}}}}$-\klshort{}:\\~\\
Using the inequality \pref{eq:glm_quadratic} from the preceeding \klshort{} proof, we have
\[
 \tri*{\grad{}\poprisk(w),w-w^{\star}} \geq{} 2c_{\sigma}^{2}\tri*{w-w^{\star},\En\brk*{xx^{\trn}}(w-w^{\star})}.
\]
By the assumption that $\nrm*{w}_{1}\leq{}\nrm*{w^{\star}}_{1}$, we apply \pref{lem:covariance_restricted_eigenvalue} to conclude that 1) $w-w^{\star}\in\cC(S(w^{\star}), 1)$ and 2) $\nrm*{w-w^{\star}}_{1}\leq{}2\sqrt{s}\nrm*{w-w^{\star}}_{2}$. The first fact implies that
\[
\tri*{w-w^{\star},\En\brk*{xx^{\trn}}(w-w^{\star})} \geq{} \eigminr(\En\brk*{xx^{\trn}})\nrm*{w-w^{\star}}_{2}^{2}.
\]
Rearranging, we have
\begin{align*}
\nrm*{w-w^{\star}}_{2} &\leq{} \frac{1}{2c_{\sigma}^{2}\eigminr(\En\brk*{xx^{\trn}})}\frac{ \tri*{\grad{}\poprisk(w),w-w^{\star}}}{\nrm*{w-w^{\star}}_{2}} \\
&\leq{} \frac{1}{2c_{\sigma}^{2}\eigminr(\En\brk*{xx^{\trn}})}\frac{ \nrm*{\grad{}\poprisk(w)}_{\infty}\nrm*{w-w^{\star}}_1}{\nrm*{w-w^{\star}}_{2}} \\
&\leq{} \frac{\sqrt{s}}{c_{\sigma}^{2}\eigminr(\En\brk*{xx^{\trn}})}\nrm*{\grad{}\poprisk(w)}_{\infty}.
\end{align*}

On the other hand, from \pref{eq:glm_kl_main} we have
\begin{align*}
\poprisk(w) - \poprisk(w^{\star}) &\leq{} \frac{C_{\sigma}}{2c_{\sigma}}\tri*{\grad{}\poprisk(w),w-w^{\star}} \\
&\leq{} \frac{C_{\sigma}}{2c_{\sigma}}\nrm*{\grad{}\poprisk(w)}_{\infty}\nrm*{w-w^{\star}}_{1} \\
&\leq{} \frac{C_{\sigma}\sqrt{s}}{c_{\sigma}}\nrm*{\grad{}\poprisk(w)}_{\infty}\nrm*{w-w^{\star}}_{2}.
\end{align*}
Combining this with the preceding inequality yields the result.
\end{itemize}

\end{proof}

The following utility lemma is a standard result in high-dimensional statistics (e.g. \cite{tibshirani2015statistical}).
\begin{lemma}
\label{lem:covariance_restricted_eigenvalue}
Let $w,w^{\star}\in\bbR^{d}$. If $\nrm*{w}_{1}\leq{}\nrm*{w^{\star}}_{1}$ then $w-w^{\star}\rdef\nu\in\cC(S(w^{\star}),1)$. Furthermore, $\nrm*{\nu}_{1}\leq{}2\sqrt{\abs*{S(w^{\star})}}\nrm*{\nu}_{2}$.
\end{lemma}
\begin{proof}[\pfref{lem:covariance_restricted_eigenvalue}]
Let $S\ldef{}S(w^{\star})$. Then the constraint that $\nrm*{w}_1\leq{}\nrm*{w^{\star}}$ implies
\begin{align*}
\nrm*{w^{\star}}_1\geq{}\nrm*{w}_{1} 
=\nrm*{w^{\star} + \nu}_{1} 
=\nrm*{w^{\star} + \nu_{S}}_{1} + \nrm*{\nu_{S^{C}}}_{1} 
\geq\nrm*{w^{\star}}_{1} - \nrm*{\nu_{S}}_{1} + \nrm*{\nu_{S^{C}}}_{1}. 
\end{align*}
Rearranging, this implies $\nrm*{\nu_{S^{C}}}_{1}\leq{}\nrm*{\nu_{S}}_{1}$, so the first result is established.

For the second result, $\nu\in\cC(S,1)$ implies $\nrm*{\nu}_1=\nrm*{\nu_{S}}_{1} + \nrm*{\nu_{S^C}}_{1}\leq{}2\nrm*{\nu_{S}}_{1}\leq{}2\sqrt{\abs*{S}}\nrm*{\nu_{S}}_{2}\leq{}2\sqrt{\abs*{S}}\nrm*{\nu}_{2}$.
\end{proof}

\begin{lemma}
\label{lem:glm_gradient}
Let the norm $\nrm*{\cdot}$ satisfy the smoothness property of \pref{thm:smooth_type} with constant $\beta$.
Then the empirical loss gradient for the generalized linear model setting enjoys the normed Rademacher complexity bound,
\begin{equation}
\En_{\eps}\sup_{w\in\cW}\nrm*{\sum_{t=1}^{n}\eps_{t}\grad{}\ls(w\midsem{}x_t,y_t)}
\leq{} O\prn*{BR^{2}C_{\sigma}^{2}\sqrt{\beta{}n}}.
\end{equation}
\end{lemma}
\begin{proof}[\pfref{lem:glm_gradient}]

Let $G_{t}(s) = \prn*{\sigma(s)-y_t}^{2}$ and $F_{t}(w)=\tri*{w,x_t}$, so that $\ls(w\midsem{}x_t,y_t)=G_t(F_t(w))$.

Observe that $G'_t(s)=2\prn*{\sigma(s)-y_t}\sigma'(s)$ and $\grad{}F_{t}(w)=x_t$, so our assumptions imply that that $\abs*{G'_t(s)}\leq{}2C_{\sigma}$ and $\nrm*{\grad{}F_t(w)}\leq{}R$. We can thus apply \pref{thm:chain_rule} to conclude
\[
\En_{\eps}\sup_{w\in\cW}\nrm*{\sum_{t=1}^{n}\eps_{t}\grad{}\ls(w\midsem{}x_t,y_t)}
\leq{} 2R\En_{\eps}\sup_{w\in\cW}\sum_{t=1}^{n}\eps_{t}G'_{t}(\tri*{w,x_t})
 + 4C_{\sigma}\En_{\eps}\nrm*{\sum_{t=1}^{n}\eps_tx_t}.
 \]
 For the first term on the left-hand side, observe that for any $s$, $\abs*{G''_t(s)}\leq{} 2\abs*{\sigma''(s)} + 2\abs*{\sigma'(s)}^{2}\leq{} 4C_{\sigma}^{2}$, so $G'_t$ is $4C_{\sigma}^{2}$-Lipschitz. The classical scalar Lipschitz contraction inequality for Rademacher complexity (\pref{lem:scalar_contraction}) therefore implies
 \[
 \En_{\eps}\sup_{w\in\cW}\sum_{t=1}^{n}\eps_{t}G'_{t}(\tri*{w,x_t}) \leq{} 
  4C_{\sigma}^{2}\En_{\eps}\sup_{w\in\cW}\sum_{t=1}^{n}\eps_{t}\tri*{w,x_t}
  = 4C_{\sigma}^{2}B\En_{\eps}\nrm*{\sum_{t=1}^{n}\eps_{t}x_t}.
 \]
 
 Finally, by our smoothness assumption on the norm, \pref{thm:smooth_type} implies
 \[
 \En_{\eps}\nrm*{\sum_{t=1}^{n}\eps_tx_t}\leq{}\sqrt{2\beta{}R^{2}n}.
 \]
\end{proof}

\subsection{Robust Regression}

\begin{proof}[\pfref{thm:rr_main}]
This proof follows the same template as \pref{thm:glm_main}. We use \pref{prop:kl_unif} and \pref{prop:kl_rademacher} to conclude that whenever $(\alpha,\mu)$-PL holds, with probability at least $1-\delta$ over the examples $\crl*{(x_t,y_t)}_{t=1}^{n}$, any learning algorithm $\walg$ satisfies
\begin{equation}
\label{eq:rr_uniform}
\poprisk(\walg) - \optrisk \leq{} c\cdot{}\mu\prn*{\nrm*{\grad\emprisk(\walg)}^{\alpha} + \prn*{\frac{\rad_{\nrm*{\cdot}}(\grad{}\ls\circ\cW\midsem{}\xr[n],\yr[n])}{n} + C_{\rho}R\sqrt{\frac{\log(1/\delta)}{n}}}^{\alpha}},
\end{equation}
where $c>0$ is an absolute constant and we have used that $\nrm*{\grad{}\ls(w\midsem{}x_t,y_t)}\leq{}C_{\rho}R$ with probability $1$.

\paragraph{Smooth high-dimensional setup} For the general smooth norm pair setup in \pref{eq:rr_uniform},  \pref{lem:rr_kl} and \pref{lem:rr_gradient} imply

\begin{align*}
\poprisk(\walg) - \optrisk &\leq{} c\cdot{}\frac{BC_{\rho}}{c_{\rho}}\prn*{\nrm*{\grad\emprisk(\walg)} + \prn*{BR^{2}C_{\rho}\sqrt{\frac{\beta{}}{n}} + C_{\rho}R\sqrt{\frac{\log(1/\delta)}{n}}}} \\
&=  \muhigh\cdot\nrm*{\grad\emprisk(\walg)} + \frac{\Chigh}{\sqrt{n}}.
\end{align*}
Where we recall $\Chigh=c\cdot\frac{B^{2}R^{2}C_{\rho}^{2}\sqrt{\beta}+C_{\rho}^{2}BR\sqrt{\log(1/\delta)}}{c_{\rho}}$ and $\muhigh=c\cdot{}\frac{BC_{\rho}}{c_{\rho}}$.

\paragraph{Low-dimensional $\ls_2/\ls_2$ setup} For the low-dimension $\ls_2/\ls_2$ pair setup in \pref{eq:rr_uniform},  \pref{lem:rr_kl} and \pref{lem:rr_gradient} imply

\begin{align*}
\poprisk(\walg) - \optrisk &\leq{} c\cdot{}\frac{C_{\rho}}{2c_{\rho}^{2}\eigmin\prn*{\Sigma}}\prn*{\nrm*{\grad\emprisk(\walg)}^{2} + \prn*{BR^{2}C_{\rho}\sqrt{\frac{1}{n}} + C_{\rho}R\sqrt{\frac{\log(1/\delta)}{n}}}^2} \\
&=  \frac{\mulow}{\eigmin(\Sigma)}\cdot\nrm*{\grad\emprisk(\walg)}^{2} + \frac{\Clow}{n\cdot\eigmin(\Sigma)},
\end{align*}
where we have used that the $\ls_2$ norm is $1$-smooth in \pref{lem:glm_gradient}. Recall that
$\Clow = c\cdot{}\frac{C_{\rho}^{3}R^{4}B^{2} + C_{\rho}^{3}R^{2}\log(1/\delta)}{c_{\rho}^{2}}$ and $\mulow=c\cdot{}\frac{C_{\rho}}{2c_{\rho}^{2}}$.

\paragraph{Sparse $\ls_{\infty}/\ls_1$ setup} For the sparse $\ls_{\infty}/\ls_1$ pair setup in \pref{eq:rr_uniform},  \pref{lem:rr_kl} and \pref{lem:rr_gradient} imply
\begin{align*}
\poprisk(\walg) - \optrisk &\leq{} c\cdot{}\frac{2C_{\rho}s}{c_{\rho}^{2}\eigminr\prn*{\Sigma}}\prn*{\nrm*{\grad\emprisk(\walg)}^{2} + \prn*{BR^{2}C_{\rho}\sqrt{\frac{\log{}d}{n}} + C_{\rho}R\sqrt{\frac{\log(1/\delta)}{n}}}^2} \\
&=  \frac{\musparse\cdot{}s}{\eigminr(\Sigma)}\cdot\nrm*{\grad\emprisk(\walg)}^{2} + \frac{s}{n}\cdot\frac{\Csparse}{\eigminr(\Sigma)},
\end{align*}
where we have used that the $\ls_{\infty}$ norm has the smoothness property with $\beta=O(\log(d))$ in \pref{lem:glm_gradient}. Recall that
$\Csparse = c\cdot{}\frac{4C_{\rho}^{3}R^{4}B^{2}\log(d) + 4C_{\rho}^{3}R^{2}\log(1/\delta)}{c_{\rho}^{2}}$ and $\musparse=c\cdot{}\frac{2C_{\rho}}{c_{\rho}^{2}}$.

\end{proof}

\begin{lemma}[\klshort{} condition for robust regression]
\label{lem:rr_kl}
Consider the robust regression setup of \pref{sec:smooth}.
\begin{itemize}[leftmargin=*]
\item When $\nrm*{\cdot}/\nrm*{\cdot}_{\star}$ are any dual norm pair, we have $\prn*{1, \frac{BC_{\rho}}{c_{\rho}}}$-\klshort{}:
\begin{equation}
\poprisk(w) - \poprisk(w^{\star}) \leq{} \frac{BC_{\rho}}{c_{\rho}}\nrm*{\grad{}\poprisk(w)}\quad\forall{}w\in\cW.
\end{equation}
\item In the $\ls_{2}/\ls_{2}$ setup, we have $\prn*{2,\frac{C_{\rho}}{2c_{\rho}^{2}\eigmin\prn*{\Sigma}}}$-\klshort{}:
\begin{equation}
\poprisk(w) - \poprisk(w^{\star}) \leq{} \frac{C_{\rho}}{2c_{\rho}^{2}\eigmin\prn*{\Sigma}}\nrm*{\grad{}\poprisk(w)}_{2}^{2}\quad\forall{}w\in\cW.
\end{equation}
\item In the sparse $\ls_{\infty}/\ls_{1}$ setup, where $\nrm*{w^{\star}}_{0}\leq{}s$, we have $\prn*{2,\frac{2C_{\rho}s}{c_{\rho}^{2}\eigminr\prn*{\Sigma}}}$-\klshort{}:
\begin{equation}
\poprisk(w) - \poprisk(w^{\star}) \leq{} \frac{2C_{\rho}s}{c_{\rho}^{2}\eigminr\prn*{\Sigma}}\nrm*{\grad{}\poprisk(w)}_{\infty}^{2}\quad\forall{}w\in\cW.
\end{equation}

\end{itemize}

\end{lemma}
\begin{proof}[\pfref{lem:rr_kl}]~\\
\textbf{Excess risk upper bound.}
To begin, smoothness of $\rho$ implies that for any $s,s^{\star}\in\cS$ we have
\[
\rho(s) - \rho(s^{\star}) \leq{} 
\leq{} \rho'(s^{\star})(s-s^{\star}) + \frac{C_{\rho}}{2}\prn*{s-s^{\star}}^{2}.
\]
Since this holds point-wise, we use it to derive the following in-expectation bound
\begin{align*}
\poprisk(w) - \poprisk(w^{\star}) 
&\leq{} \En_{x,y}\brk*{\rho(\tri*{w^{\star},x}-y)\tri*{w-w^{\star},x}} + \frac{C_{\rho}}{2}\En\tri*{w-w^{\star},x}^{2} \\
&= \tri*{\grad\poprisk(w^{\star}),w-w^{\star}} + \frac{C_{\rho}}{2}\En\tri*{w-w^{\star},x}^{2}.
\end{align*}
Note however that
\[
\grad\poprisk(w^{\star}) = \En_{x,\zeta}\brk*{\rho'(-\zeta)x} = 0,
\]
since $\zeta$ is conditionally symmetric and $\rho'$ is odd. We therefore have
\[
\poprisk(w) - \poprisk(w^{\star}) \leq{} \frac{C_{\rho}}{2}\En\tri*{w-w^{\star},x}^{2}.
\]
On the other hand, using the form of the gradient we have
\begin{align*}
\tri*{\poprisk(w),w-w^{\star}} &= \En_{x}\brk*{\En_{\zeta}\rho'(\tri*{w-w^{\star},x}-\zeta)\tri*{w-w^{\star},x}} \\
&= \En_{x}\brk*{h(\tri*{w-w^{\star},x})\tri*{w-w^{\star},x}}.
\end{align*}
To lower bound the term inside the expectation, consider a particular draw of $x$ and assume $\tri*{w-w^{\star},x}\geq{}0$; this is admissible because $h$, like $\rho'$, is odd. Then we have
\[
h(\tri*{w-w^{\star},x})\tri*{w-w^{\star},x} = \frac{h(\tri*{w-w^{\star},x})}{\tri*{w-w^{\star},x}}\tri*{w-w^{\star},x}^{2}
\geq{}  c_{\rho}\tri*{w-w^{\star},x}^{2},
\]
where the last line follows because $h(0)=0$ and $h'(0)>c_{\rho}$. Since this holds pointwise, we simply take the expectation to show that
\[
\tri*{\grad\poprisk(w),w-w^{\star}} \geq{} c_{\rho}\En_{x}\tri*{w-w^{\star},x}^{2},\numberthis\label{eq:rr_data_norm}
\]
and consequently the excess risk is bounded by
\[
\poprisk(w) - \poprisk(w^{\star}) \leq{} \frac{C_{\rho}}{2c_{\rho}}\tri*{\grad{}\poprisk(w),w-w^{\star}}.
\numberthis\label{eq:kl_rr}
\]

\textbf{Proving the \klshort{} conditions.}
We now use \pref{eq:kl_rr} to establish the \klshort{} condition variants.
\begin{itemize}[leftmargin=*]
\item $\prn*{1, \frac{BC_{\rho}}{c_{\rho}}}$-\klshort{}: \\~\\
Use \Holder{}'s inequality to obtain the upper bound,
\[
\tri*{\grad{}\poprisk(w),w-w^{\star}} \leq{} 2B\nrm*{\grad{}\poprisk(w)}.
\]
\item $\prn*{2,\frac{C_{\rho}}{2c_{\rho}^{2}\eigmin\prn*{\Sigma}}}$-\klshort{}:\\~\\
Begin with
\begin{align*}
\poprisk(w) - \poprisk(w^{\star}) &\leq{} \frac{C_{\rho}}{2c_{\rho}}\tri*{\grad{}\poprisk(w),w-w^{\star}}.
\intertext{Using the same reasoning as in \pref{lem:glm_kl}, this is upper bounded by}
&\leq{} \frac{C_{\rho}}{2c_{\rho}}\nrm*{\grad{}\poprisk(w)}_{2}\cdot\nrm*{P_{\cX}\prn*{w-w^{\star}}}_{2}, \numberthis\label{eq:rr_norm}
\end{align*}
where $P_{\cX}$ denotes the orthogonal projection onto $\mathrm{span}(\Sigma)$.

Recalling \pref{eq:rr_data_norm}, it also holds that 
\begin{align*}
\tri*{P_{\cX}\grad{}\poprisk(w),P_{\cX}\prn*{w-w^{\star}}} = \tri*{\grad{}\poprisk(w),w-w^{\star}} &\geq{} c_{\rho}\En_x\tri*{w-w^{\star},x}^{2} \\
&= c_{\rho}\tri*{w-w^{\star}, \En_x \brk*{xx^T} \prn*{w - w^*}} \\
&= c_{\rho}\tri*{w-w^{\star},\Sigma(w-w^{\star})} \numberthis\label{eq:rr_quadratic}\\
&\geq{} c_{\rho}\eigmin\prn*{\Sigma}\nrm*{P_{\cX}(w-w^{\star})}_{2}^{2}.
\end{align*}
Rearranging and applying Cauchy-Schwarz, we have
\[
\nrm*{P_{\cX}(w-w^{\star})}_{2} \leq{} \frac{1}{c_{\rho}\eigmin\prn*{\Sigma}}\cdot\nrm*{\grad{}\cL_{\cD}(w)}_{2}.
\]
Combining this inequality with \pref{eq:rr_norm}, we have
\[
\poprisk(w) - \poprisk(w^{\star}) \leq{} \frac{C_{\rho}}{2c_{\rho}^{2}\eigmin\prn*{\Sigma}}\cdot\nrm*{\grad{}\cL_{\cD}(w)}_{2}^{2}.
\]
\item $\prn*{2,\frac{2C_{\rho}s}{c_{\rho}^{2}\eigminr\prn*{\Sigma}}}$-\klshort{}:\\~\\
Using the inequality \pref{eq:rr_quadratic} from the $\ls_2/\ls_2$ \klshort condition proof above
\[
 \tri*{\grad{}\poprisk(w),w-w^{\star}} \geq{} c_{\rho}\tri*{w-w^{\star},\Sigma(w-w^{\star})}.
\]
By the assumption that $\nrm*{w}_{1}\leq{}\nrm*{w^{\star}}_{1}$, we apply \pref{lem:covariance_restricted_eigenvalue} to conclude that 1) $w-w^{\star}\in\cC(S(w^{\star}), 1)$ and 2) $\nrm*{w-w^{\star}}_{1}\leq{}2\sqrt{s}\nrm*{w-w^{\star}}_{2}$, and so
\[
\tri*{w-w^{\star},\Sigma(w-w^{\star})} \geq{} \eigminr(\Sigma)\nrm*{w-w^{\star}}_{2}^{2}.
\]
Rearranging, and applying the $\nrm*{w-w^{\star}}_{1}\leq{}2\sqrt{s}\nrm*{w-w^{\star}}_{2}$ inequality:
\begin{align*}
\nrm*{w-w^{\star}}_{2} &\leq{} \frac{1}{c_{\rho}\eigminr(\Sigma)}\frac{ \tri*{\grad{}\poprisk(w),w-w^{\star}}}{\nrm*{w-w^{\star}}_{2}} \\
&\leq{} \frac{1}{c_{\rho}\eigminr(\Sigma)}\frac{ \nrm*{\grad{}\poprisk(w)}_{\infty}\nrm*{w-w^{\star}}_1}{\nrm*{w-w^{\star}}_{2}} \\
&\leq{} \frac{2\sqrt{s}}{c_{\rho}\eigminr(\Sigma)}\nrm*{\grad{}\poprisk(w)}_{\infty}.
\end{align*}

Finally, from \pref{eq:kl_rr} we have
\begin{align*}
\poprisk(w) - \poprisk(w^{\star}) &\leq{} \frac{C_{\rho}}{2c_{\rho}}\tri*{\grad{}\poprisk(w),w-w^{\star}} \\
&\leq{} \frac{C_{\rho}}{2c_{\rho}}\nrm*{\grad{}\poprisk(w)}_{\infty}\nrm*{w-w^{\star}}_{1} \\
&\leq{} \frac{C_{\rho}\sqrt{s}}{c_{\rho}}\nrm*{\grad{}\poprisk(w)}_{\infty}\nrm*{w-w^{\star}}_{2}.
\end{align*}
Combining the two inequalities gives the final result.
\end{itemize}

\end{proof}

\begin{lemma}
\label{lem:rr_gradient}
Let the norm $\nrm*{\cdot}$ satisfy the smoothness property (see \pref{thm:smooth_type}) with constant $\beta$.
Then the gradient for robust regression satisfies the following normed Rademacher complexity bound:
\begin{equation}
\En_{\eps}\sup_{w\in\cW}\nrm*{\sum_{t=1}^{n}\eps_{t}\grad{}\ls(w\midsem{}x_t,y_t)}
\leq{} O\prn*{BR^{2}C_{\rho}\sqrt{\beta{}n}}.
\end{equation}
\end{lemma}
\begin{proof}[\pfref{lem:rr_gradient}]

Let $G_{t}(s) = \rho(s-y_t)$ and $F_{t}(w)=\tri*{w,x_t}$, so that $\ls(w\midsem{}x_t,y_t)=G_t(F_t(w))$.

Then  $G'_t(s)=\rho'(s-y_t)$ and $\grad{}F_{t}(w)=x_t$, so our assumptions imply that that $\abs*{G'_t(s)}\leq{}C_{\rho}$ and $\nrm*{\grad{}F_t(w)}\leq{}R$. We apply \pref{thm:chain_rule} to conclude
\[
\En_{\eps}\sup_{w\in\cW}\nrm*{\sum_{t=1}^{n}\eps_{t}\grad{}\ls(w\midsem{}x_t,y_t)}
\leq{} 2R\En_{\eps}\sup_{w\in\cW}\sum_{t=1}^{n}\eps_{t}G'_{t}(\tri*{w,x_t})
 + 2C_{\rho}\En_{\eps}\nrm*{\sum_{t=1}^{n}\eps_tx_t}.
 \]
 For the first term on the left-hand side, we have that for any $s$, $\abs*{G''_t(s)}=2\abs*{\rho''(s-y_t)}\leq{}2C_{\rho}$, so $G'_t$ is $2C_{\sigma}$-Lipschitz. Then the by scalar contraction for Rademacher complexity (\pref{lem:scalar_contraction}),
 \[
 \En_{\eps}\sup_{w\in\cW}\sum_{t=1}^{n}\eps_{t}G'_{t}(\tri*{w,x_t}) \leq{} 
2C_{\rho}\En_{\eps}\sup_{w\in\cW}\sum_{t=1}^{n}\eps_{t}\tri*{w,x_t}
  = 2C_{\rho}B\En_{\eps}\nrm*{\sum_{t=1}^{n}\eps_{t}x_t}.
 \]
 
 Finally, the smoothness assumption on the norm (via \pref{thm:smooth_type}) implies
 \[
 \En_{\eps}\nrm*{\sum_{t=1}^{n}\eps_tx_t}\leq{}\sqrt{2\beta{}R^{2}n}.
 \]
\end{proof}

\begin{proof}[\pfref{prop:glm_opt}]
Observe that \pref{ass:glm} and \pref{ass:rr} respectively imply that $\nrm*{\grad\poprisk(\wopt)}=0$ for the GLM and RR settings.
Begin by invoking  \pref{thm:glm_main}. It is immediate that any algorithm that guarantees $\En\nrm*{\grad{}\emprisk(\walg)}\leq{}1/\sqrt{n}$ will obtain the claimed sample complexity bound (the high-probability statement \pref{thm:glm_main} immediately yields an in-expectation statement due to boundedness), so all we must do is verify that such a point exists. \pref{prop:kl_rademacher} along with \pref{lem:glm_gradient} and \pref{lem:rr_gradient} respectively indeed imply that $\nrm*{\grad\emprisk(w^{\star})}_{2}\leq{}C/\sqrt{n}$ for both settings.

For completeness, we show below that both models indeed have Lipschitz gradients, and so standard smooth optimizers can be applied to the empirical loss.

\emph{Generalized Linear Model.}~~ Observe that for any $(x,y)$ pair we have
\begin{align*}
\nrm*{\grad{}\ls(w\midsem{}x,y)-\grad{}\ls(w'\midsem{}x,y)}_{2} &= 
2\nrm*{x}_{2}\abs*{(\sigma(\tri*{w,x})-y)\sigma'(\tri*{w,x}) - (\sigma(\tri{w',x})-y)\sigma'(\tri{w',x})}.
\end{align*}
Letting $f(s)=(\sigma(s)-y)\sigma'(s)$, we see that the assumption on the loss guarantees $\abs*{f'(s)}\leq{}3C_{\sigma}^{2}$, so we have
\[
\nrm*{\grad{}\ls(w\midsem{}x,y)-\grad{}\ls(w'\midsem{}x,y)}_{2}
\leq{} 6C_{\sigma}^{2}R\abs*{\tri*{w-w',x}}
\leq{} \leq{} 6C_{\sigma}^{2}R^{2}\nrm*{w-w'}_{2},
\]
so smoothness is established.

\emph{Robust Regression.}~~ Following a similar calculation to the GLM case, we have
\begin{align*}
\nrm*{\grad{}\ls(w\midsem{}x,y)-\grad{}\ls(w'\midsem{}x,y)}_{2} &= 
\nrm*{x}_{2}\abs*{\rho'(\tri*{w,x}-y) - \rho'(\tri{w',x}-y)} \\
&\leq{} 
C_{\rho}\nrm*{x}_{2}\abs*{\tri*{w-w^{\star},x}} \\
&\leq{} 
C_{\rho}\nrm*{x}_{2}^{2}\nrm*{w-w^{\star}}_{2} \\
&\leq{} 
C_{\rho}R^{2}\nrm*{w-w^{\star}}_{2}.
\end{align*}
Now let $f(s)=(\sigma(s)-y)\sigma'(s)$, and observe that $\abs*{f'(s)}\leq{}3C_{\sigma}^{2}$, so we have
\[
\nrm*{\grad{}\ls(w\midsem{}x,y)-\grad{}\ls(w'\midsem{}x,y)}_{2}
\leq{} 6C_{\sigma}^{2}R\abs*{\tri*{w-w',x}}
\leq{} \leq{} 6C_{\sigma}^{2}R^{2}\nrm*{w-w'}_{2}.
\]
\end{proof}

\subsection{Further Discussion}
\label{app:further_discussion}

\paragraph{Detailed comparison with \cite{mei2016landscape}}
We now sketch in more detail the relation between the rates of \pref{thm:glm_main} and \pref{thm:rr_main} and those of \cite{mei2016landscape}. We focus on the fast rate regime, and on the case $R=\sqrt{d}$ (e.g., when $x\sim{}\cN(0,I_{d\times{}d})$).
\begin{itemize}
\item \emph{Uniform convergence.} Their uniform convergence bounds scale as $O\prn{\tau\sqrt{d/n}}$, where $\tau$ is the subgaussian parameter for the data $x$, whereas our uniform convergence bounds scale as $O\prn{R^{2}\sqrt{1/n}}$. When $R=\sqrt{d}$ both bounds scale as $O\prn{d\sqrt{1/n}}$, but our bounds do not depend on $d$ when $R$ is constant, whereas their bound always pays $\sqrt{d}$.

\item \emph{Parameter convergence.} The final result of \cite{mei2016landscape} is a parameter convergence bound of the form $\| \walg - w^* \|_2 \leq O\prn*{\frac{\tau}{\underline{\gamma} \tau^2}  \sqrt{\frac{d}{n}}}$ (see Theorem 4/6; Eqs. (106) and (96)). Our main result for the ``low-dimensional'' setup in \pref{thm:glm_main} and \pref{thm:rr_main} is an excess risk bound of the form $ \poprisk(\walg) - \poprisk(w^*) \leq O\prn*{\frac{R^4}{\lambda_{\mathrm{min}}(\Sigma)n}}$ which implies a parameter convergence bound of $\nrm*{\walg - w^*}_2 \leq \frac{R^2}{\lambda_{\mathrm{min}}(\Sigma) \sqrt{n}}$ (using similar reasoning as in the proof of \pref{lem:glm_kl} and \pref{lem:rr_kl}). With $\tau = R = \sqrt{d}$ and Assumptions 6 and 9 in \cite{mei2016landscape}, we have $\lambda_{\mathrm{min}}(\Sigma) = \underline{\gamma} \tau^2$, and so again both the bounds resolve to $O\prn*{\frac{d}{\lambda_{\mathrm{min}}(\Sigma) \sqrt{n}}}$.
\end{itemize}

\paragraph{Analysis of regularized stationary point finding for high-dimensional setting}
Here we show that any algorithm that finds a stationary point of the regularized empirical loss generically succeeds obtains optimal sample complexity in the high-dimensional/norm-based setting. We focus on the generalized linear model in the Euclidean setting.

Let $r(w) = \frac{\lambda}{2}\nrm*{w}^{2}_{2}$. Define $\poprisk^{\lambda}(w) = \poprisk(w) + r(w)$ and $\emprisk^{\lambda}(w) = \emprisk(w) + r(w)$. We consider any algorithm that returns a point $\wh{w}$ with $\grad{}\emprisk^{\lambda}(\wh{w})=0$, i.e. any stationary point of the regularized empirical risk.

\begin{theorem}
\label{thm:regularized_glm}
Consider the generalized linear model setting. Let $\wh{w}$ be any point with $\grad{}\emprisk^{\lambda}(\wh{w})=0$. Suppose that $\nrm*{w^{\star}}_{2}=1$ and $C_{\sigma},R>1$.
Then there is some absolute constant $c>0$ such that for any fixed $\delta>0$,  if the regularization parameter $\lambda$ satisfies
\[
\lambda>c\cdot\sqrt{\frac{R^{4}C_{\sigma}^{6}}{c_{\sigma}^{2}}\cdot\frac{\log\prn*{\log{}\prn*{C_{\sigma}Rn}/\delta}}{n}},
\] then with probability at least $1-\delta$,
\[
\poprisk(\wh{w}) - \poprisk(w^{\star}) \leq{} O\prn*{
\frac{R^{2}C_{\sigma}^{4}}{c_{\sigma}^{2}}\cdot\sqrt{\frac{\log\prn*{\log{}\prn*{C_{\sigma}Rn}/\delta}}{n}}
}.
\]
\end{theorem}
\pref{thm:regularized_glm} easily extends to the robust regression setting by replacing invocations of \pref{lem:glm_gradient} with \pref{lem:rr_gradient} and use of \pref{eq:glm_kl_main} with \pref{eq:kl_rr}.
\begin{proof}[\pfref{thm:regularized_glm}]
Recall that  $w^{\star}$ minimizes the \emph{unregularized} population risk, and that $\nrm*{w^{\star}}_{2}=1$.  The technical challenge is to apply \pref{lem:glm_gradient} even though we lack a good a-priori upper bound on the norm of $\wh{w}$. We proceed by splitting the analysis into two cases. The idea is that if $\nrm*{\wh{w}}_{2}\leq{}\nrm*{w^{\star}}_2$ we can apply \pref{lem:glm_gradient} directly with no additional difficulty. On other hand, when $\nrm*{\wh{w}}_{2}\geq{}\nrm*{w^{\star}}_2$ the regularized population risk satisfies the $(2,O(1/\lambda))$-\klshort{} inequality, which is enough to show that excess risk is small even though $\nrm*{\wh{w}}_2$ could be larger than $\nrm*{w^{\star}}_{2}$.
\\~\\
\textbf{Case 1: $\nrm*{\wh{w}}_{2}\geq{}\nrm*{w^{\star}}_2$.}~\\~\\ Let $\tW=\crl*{w\in\bbR^{d}\mid\nrm*{w}_2\geq{}\nrm*{w^{\star}}_2}$, so that $\wh{w}\in\tW$. Observe that since $r(w)$ is $\lambda$-strongly convex it satisfies $r(w) - r(w^{\star})\leq{}\tri*{\grad{}r(w),w-w^{\star}}-\frac{\lambda}{2}\nrm*{w-w^{\star}}_2^{2}$ for all $w$. Moreover, if $w\in\tW$, we have
\[
\tri*{\grad{}r(w),w-w^{\star}}\geq{}r(w)-r(w^{\star}) +\frac{\lambda}{2}\nrm*{w-w^{\star}}_{2}^{2} \geq{} 0.
\]
Using \pref{eq:glm_kl_main} and the definition of $w^{\star}$, along with the strong convexity of $r$, we get
\[
\poprisk^{\lambda}(w) - \poprisk^{\lambda}(w^{\star}) \leq{} \frac{C_{\sigma}}{2c_{\sigma}}\tri*{\grad{}\poprisk(w),w-w^{\star}} + \tri*{\grad{}r(w),w-w^{\star}} - \frac{\lambda}{2}\nrm*{w-w^{\star}}_2^{2}.
\]
Since $\tri*{\grad\poprisk(w),w-w^{\star}}\geq{}0$, this is upper bounded by
\[
\poprisk^{\lambda}(w) - \poprisk^{\lambda}(w^{\star}) \leq{} \frac{C_{\sigma}}{c_{\sigma}}\tri*{\grad{}\poprisk(w),w-w^{\star}} + \tri*{\grad{}r(w),w-w^{\star}} - \frac{\lambda}{2}\nrm*{w-w^{\star}}_2^{2}.
\]
Using the non-negativity of $\tri*{\grad{}r(w),w-w^{\star}}$ over $\tW$, and that $C_{\sigma}/c_{\sigma}>1$, this implies
\begin{align}
\label{eq:regularized}
\poprisk^{\lambda}(w) - \poprisk^{\lambda}(w^{\star}) &\leq{} \frac{C_{\sigma}}{c_{\sigma}}\tri*{\grad{}\poprisk^{\lambda}(w),w-w^{\star}} - \frac{\lambda}{2}\nrm*{w-w^{\star}}_2^{2}\quad\forall{}w\in\tW.\notag
\intertext{Applying Cauchy-Schwarz:}
&\leq{} \frac{C_{\sigma}}{c_{\sigma}}\nrm*{\grad{}\poprisk^{\lambda}(w)}_2\nrm*{w-w^{\star}}_2 - \frac{\lambda}{2}\nrm*{w-w^{\star}}_2^{2}\quad\forall{}w\in\tW.\notag
\intertext{Using the AM-GM inequality:}
&\leq{} \frac{C_{\sigma}^{2}}{c_{\sigma}^{2}\lambda}\nrm*{\grad{}\poprisk^{\lambda}(w)}_2^{2}\quad\forall{}w\in\tW.\notag
\end{align}
Using that $\wh{w}\in\tW$, and that $\grad\emprisk^{\lambda}(\wh{w})=0$, we have
\begin{equation}
\label{eq:what_amgm}
\poprisk^{\lambda}(\wh{w}) - \poprisk^{\lambda}(w^{\star}) \leq{} \frac{C_{\sigma}^{2}}{c_{\sigma}^{2}\lambda}\nrm*{\grad{}\poprisk^{\lambda}(\wh{w})-\grad{}\emprisk^{\lambda}(\wh{w})}_2^{2}.
\end{equation}
Observe that since $\wh{w}$ is a stationary point of the empirical risk, $\grad\emprisk(\wh{w})=-\lambda\wh{w}$, and so $\nrm*{\wh{w}}_{2}\leq{}\frac{1}{\lambda}\nrm*{\grad\emprisk(\wh{w})}_{2}\leq{}\frac{2C_{\sigma}R}{\lambda}$ with probability $1$. Thus, if we apply \pref{lem:glm_forall} with $B_{\mathrm{max}}=\frac{2C_{\sigma}R}{\lambda}$, we get that with probability at least $1-\delta$,
\[
\nrm*{\grad\poprisk^{\lambda}(\wh{w})-\grad\emprisk^{\lambda}(\wh{w})}_{2} \leq{} O\prn*{
\nrm*{\wh{w}}_{2}R^{2}C_{\sigma}^{2}\sqrt{\frac{1}{n}}
+C_{\sigma}R\sqrt{\frac{\log\prn*{\log(C_{\sigma}R/\lambda)/\delta}}{n}}
},
\]
where we have used additionally that the regularization term does not depend on data. Combining this bound with \pref{eq:what_amgm}, and using that $\wh{w}\in\tW$ and the elementary inequality $(a+b)^{2}\leq{}2(a^{2}+b^{2})$, we see that there exist constants $c,c'>0$ such that
\[
\poprisk^{\lambda}(\wh{w}) - \poprisk^{\lambda}(w^{\star}) \leq{} 
c\cdot\nrm*{\wh{w}}_{2}^{2}\cdot\frac{R^{4}C_{\sigma}^{6}}{\lambda{}c_{\sigma}^{2}}\cdot\frac{1}{n}
+ 
c'\cdot\frac{R^{2}C_{\sigma}^{4}}{\lambda{}c_{\sigma}^{2}}\cdot\frac{\log\prn*{\log(C_{\sigma}R/\lambda)/\delta}}{n}.
\]

Expanding the definition of the regularized excess risk, this is equivalent to
\[
\poprisk(\wh{w}) - \poprisk(w^{\star}) \leq{} \lambda + 
\nrm*{\wh{w}}_{2}^{2}\cdot\prn*{c\cdot\frac{R^{4}C_{\sigma}^{6}}{\lambda{}c_{\sigma}^{2}}\cdot\frac{1}{n}-\lambda}
+ 
c'\cdot\frac{R^{2}C_{\sigma}^{4}}{\lambda{}c_{\sigma}^{2}}\cdot\frac{\log\prn*{\log(C_{\sigma}R/\lambda)/\delta}}{n}.
\]
Observe that if $\lambda>\sqrt{c\cdot\frac{R^{4}C_{\sigma}^{6}}{c_{\sigma}^{2}}\cdot\frac{1}{n}}$ the middle term in this expression is at most zero. We choose
\[
\lambda>\sqrt{c\cdot\frac{R^{4}C_{\sigma}^{6}}{c_{\sigma}^{2}}\cdot\frac{\log\prn*{\log{}\prn*{C_{\sigma}Rn}/\delta}}{n}}.
\]
Substituting choice this into the expression above leads to a final bound of
\[
\poprisk(\wh{w}) - \poprisk(w^{\star}) \leq{} O\prn*{
\frac{R^{2}C_{\sigma}^{3}}{c_{\sigma}}\cdot\sqrt{\frac{\log\prn*{\log{}\prn*{C_{\sigma}Rn}/\delta}}{n}}
}.
\]

\paragraph{Case 2: $\nrm*{\wh{w}}_{2}\leq\nrm*{w^{\star}}_2$.}~\\~\\
Recall that $\grad{}\emprisk^{\lambda}(\wh{w})=0$. This implies $\grad\emprisk(\wh{w})=-\lambda\wh{w}$, and so $\nrm*{\grad{}\emprisk(\wh{w})}_2\leq{}\lambda\nrm*{\wh{w}}_2\leq{}\lambda$. Using \pref{eq:glm_kl_main} we have
\[
\poprisk(\wh{w}) - \poprisk(w^{\star}) \leq{} \frac{C_{\sigma}}{2c_{\sigma}}\tri*{\grad{}\poprisk(\wh{w}),\wh{w}-w^{\star}}
\leq{} \frac{C_{\sigma}}{c_{\sigma}}\nrm*{\grad{}\poprisk(\wh{w})}_2.
\]
Using the bound on the empirical gradient above, we get
\[
\nrm*{\grad{}\poprisk(\wh{w})}_2 \leq{} \lambda + \nrm*{\grad{}\poprisk(\wh{w})-\grad\emprisk(\wh{w})}_2.
\]
Using \pref{eq:uniform_convergence}, \pref{eq:mcdiarmid}, and \pref{lem:glm_gradient}, applied with $B=1$, we have that with probability at least $1-\delta$,
\[
\nrm*{\grad{}\poprisk(\wh{w})-\grad\emprisk(\wh{w})}_2
\leq{} O\prn*{
R^{2}C_{\sigma}^{2}\sqrt{\frac{\log(1/\delta)}{n}
}},
\]
and so
\begin{align*}
\poprisk(\wh{w}) - \poprisk(w^{\star}) 
&\leq{} O\prn*{\lambda\frac{C_{\sigma}}{c_{\sigma}} + 
\frac{R^{2}C_{\sigma}^{3}}{c_{\sigma}}\sqrt{\frac{\log(1/\delta)}{n}
}}.
\intertext{Substituting in the choice for $\lambda$:}
&\leq{} O\prn*{
\frac{R^{2}C_{\sigma}^{4}}{c_{\sigma}^{2}}\cdot\sqrt{\frac{\log\prn*{\log{}\prn*{C_{\sigma}Rn}/\delta}}{n}}
}.
\end{align*}
\end{proof}
\begin{lemma}
\label{lem:glm_forall}
Let $\poprisk$ and $\emprisk$ be the population and empirical risk for the generalized linear model setting. Let a parameter $B_{\mathrm{max}}\geq{}1$ be given. Then with probability at least $1-\delta$, for all $w\in\bbR^{d}$ with $1\leq{}\nrm*{w}_{2}\leq{}B_{\mathrm{max}}$,
\[
\nrm*{\grad\poprisk(w)-\grad\emprisk(w)}_{2} \leq{} O\prn*{
\nrm*{w}_{2}R^{2}C_{\sigma}^{2}\sqrt{\frac{1}{n}}
+C_{\sigma}R\sqrt{\frac{\log\prn*{\log(B_{\mathrm{max}})/\delta}}{n}}
},
\]
where all constants are as in \pref{ass:glm}.
\end{lemma}
\begin{proof}
\pref{eq:uniform_convergence}, \pref{eq:mcdiarmid}, and \pref{lem:glm_gradient} imply that for any fixed $B$, with probability at least $1-\delta$,
\[
\sup_{w:\nrm*{w}_{2}\leq{}B}\nrm*{\grad\poprisk(w)-\grad\emprisk(w)} \leq{} O\prn*{
BR^{2}C_{\sigma}^{2}\sqrt{\frac{1}{n}}
+C_{\sigma}R\sqrt{\frac{\log\prn*{\frac{1}{\delta}}}{n}}
}.
\]
Define $B_{i}=e^{i-1}$ for $1\leq{}i\leq{}\ceil{\log(B_{\mathrm{max}})}+1$. The via a union bound, we have that for all $i$ simultaneously,
\[
\sup_{w:\nrm*{w}_{2}\leq{}B_i}\nrm*{\grad\poprisk(w)-\grad\emprisk(w)} \leq{} O\prn*{
B_{i}R^{2}C_{\sigma}^{2}\sqrt{\frac{1}{n}}
+C_{\sigma}R\sqrt{\frac{\log\prn*{\log(B_{\mathrm{max}})/\delta}}{n}}
}.
\]
In particular, for any fixed $w$ with $1\leq{}\nrm*{w}_2\leq{}B_{\mathrm{max}}$, if we take $i$ to be the smallest index for which $\nrm*{w}_{2}\leq{}B_i$, the expression above implies
\[
\nrm*{\grad\poprisk(w)-\grad\emprisk(w)}_{2} \leq{} O\prn*{
\nrm*{w}_{2}R^{2}C_{\sigma}^{2}\sqrt{\frac{1}{n}}
+C_{\sigma}R\sqrt{\frac{\log\prn*{\log(B_{\mathrm{max}})/\delta}}{n}}
},
\]
since $B_{i}\leq{}e\nrm*{w}_{2}$.
\end{proof}

\paragraph{Analysis of mirror descent for high-dimensional setting.} Here we show that mirror descent obtains optimal excess risk for the norm-based/high-dimensional regime in \pref{thm:glm_main} and \pref{thm:rr_main}.

Our approach is to run mirror descent with $\Psi^{\star}$ as the regularizer. Observe that $\Psi^{\star}$ is $\frac{1}{\beta}$-strongly convex with respect to the dual norm $\nrm*{\cdot}_{\star}$, and that we have $\nrm*{\grad{}\ls(w\midsem{}x,y)}\leq{}2C_{\sigma}R$ for the GLM setting and $\nrm*{\grad{}\ls(w\midsem{}x,y)}\leq{}C_{\rho}R$ for the RR setting. 

Focusing on the GLM, if we take a single pass over the entire dataset $\crl*{(x_t,y_t)}_{t=1}^{n}$ in order, the standard analysis for mirror descent starting at $w_1=0$ with optimal learning rate tuning \citep{hazan2016introduction} guarantees that the following inequality holds deterministically:
\[
\frac{1}{n}\sum_{t=1}^{n}\tri*{\grad{}\ls(w_t\midsem{}x_t,y_t),w_t-w^{\star}}
\leq{} O\prn*{
RBC_{\sigma}\sqrt{\frac{\beta}{n}}
}.
\]
Since each point is visited a single time, this leads to the following guarantee on the population loss in expectation
\[
\En\brk*{\frac{1}{n}\sum_{t=1}^{n}\tri*{\grad{}\poprisk(w_t),w_t-w^{\star}}}
\leq{} O\prn*{
RBC_{\sigma}\sqrt{\frac{\beta}{n}}
}.
\]
Consequently, if we define $\wh{w}$ to be the result of choosing a single time $t\in\brk*{n}$ uniformly at random and returning $w_t$, this implies that
\[
\En\brk*{\tri*{\grad{}\poprisk(\wh{w}),\wh{w}-w^{\star}}}
\leq{} O\prn*{
RBC_{\sigma}\sqrt{\frac{\beta}{n}}
}.
\]
Combining this inequality with \pref{eq:glm_kl_main}, we have
\[
\En\brk*{\poprisk(\wh{w}) - \poprisk(w^{\star})} \leq{} O\prn*{RB
\frac{C^{2}_{\sigma}}{c_{\sigma}}\sqrt{\frac{\beta}{n}}
}.
\]

Likewise, combining the mirror descent upper bound with \pref{eq:kl_rr} leads to a rate of $O\prn*{RB\frac{C_{\rho}^{2}}{c_{\rho}}\sqrt{\frac{\beta}{n}}}$ for robust regression. Thus, when all parameters involved are constant, it suffices to take $n=\frac{1}{\veps^{2}}$ to obtain $O(\veps)$ excess risk in both settings.


\section{Proofs from \pref{sec:nonsmooth}}
\label{app:nonsmooth}

\begin{proof}[\pfref{thm:relu_lb_l2_l2}]

Let $B\in\bbR^{d\times{}d}$ be a matrix for which the $i$th row $B_{i}$ is given by $B_{i}=\frac{1}{\sqrt{d}}(\ones-e_i)$.

We first focus on the more technical case where $n\geq{}d$.

Let $n = N\cdot{}d$ for some odd $N\in\bbN$. We partition time into $d$ consecutive segments: $S_{1}=\crl*{1,\ldots,N}$, $S_{2}=\crl*{N+1,\ldots,2N}$ and on. The sequence of instances $\xr[n]$ we will use will be to set $x_{t}=B_{i}$ for $t\in{}S_i$. Note that $\nrm*{B_i}_2\leq{}1$, so this choice indeed satisfies the boundedness constraint.

For simplicity, assume that $y_t=-1$ for all $t\in\brk*{n}$. 
Then it holds that
\begin{align*}
\En_{\eps}\sup_{w\in\cW}\nrm*{\sum_{t=1}^{n}\eps_{t}\grad{}\ls(w\midsem{}x_t,y_t)}_{2}
&=\En_{\eps}\sup_{w\in\cW}\nrm*{\sum_{t=1}^{n}\eps_{t}\ind\crl*{\tri*{w,x_t}\geq{}0}x_t}_{2} \\
&=\En_{\eps}\sup_{w\in\cW}\nrm*{\sum_{i=1}^{d}\ind\crl*{\tri*{w,B_i}\geq{}0}\sum_{t\in{}S_i}\eps_{t}x_t}_{2}
\intertext{We introduce the notation $\vphi_i=\sum_{t\in{}S_i}\eps_t$.}
&=\En_{\vphi}\sup_{w\in\cW}\nrm*{\sum_{i=1}^{d}\ind\crl*{\tri*{w,B_i}\geq{}0}\vphi_i{}B_i}_{2}\\
&=\En_{\vphi}\sup_{w\in\cW}\nrm*{\sum_{i=1}^{d}\ind\crl*{\tri*{w,B_i}\geq{}0}\vphi_i\frac{1}{\sqrt{d}}(\ones-e_i)}_{2}
\intertext{Using triangle inequality:}
&\geq\En_{\vphi}\sup_{w\in\cW}\nrm*{\sum_{i=1}^{d}\ind\crl*{\tri*{w,B_i}\geq{}0}\vphi_i\frac{1}{\sqrt{d}}\ones}_{2} - \frac{1}{\sqrt{d}}\En_{\vphi}\sum_{i=1}^{d}\abs*{\vphi_i} \\
&=\En_{\vphi}\sup_{w\in\cW}\abs*{\sum_{i=1}^{d}\ind\crl*{\tri*{w,B_i}\geq{}0}\vphi_i} - \frac{1}{\sqrt{d}}\En_{\vphi}\sum_{i=1}^{d}\abs*{\vphi_i}\\
&\geq\En_{\vphi}\sup_{w\in\cW}\abs*{\sum_{i=1}^{d}\ind\crl*{\tri*{w,B_i}\geq{}0}\vphi_i} - O(\sqrt{n}).
\end{align*}
Now, for a given draw of $\vphi$, we choose $w\in\cW$ such that $\sgn(\tri*{w,B_i})=\sgn(\vphi_i)$. The key trick here is that $B$ is invertible, so for a given sign pattern, say $\sigma\in\pmo^{d}$, we can set $\wt{w}=B^{-1}\sigma$ and then $w=\wt{w}/\nrm*{\wt{w}}_{2}$ to achieve this pattern. To see that $B$ is invertible, observe that we can write it as $B=\frac{1}{\sqrt{d}}\prn*{\ones\ones^{\trn} - I}$. The identity matrix can itself be written as $\frac{1}{d}\ones\ones^{\trn} + A_{\perp}$, where $\ones\notin\mathrm{span}(A_{\perp})$, so it can be seen that $B=\frac{1}{\sqrt{d}}\prn*{(1-\frac{1}{d})\ones\ones^{\trn} - A_{\perp}}$, and that the $\ones\ones^{\trn}$ component is preserved by this addition.

We have now arrived at a lower bound of $\En_{\vphi}\abs*{\sum_{i=1}^{d}\ind\crl*{\sgn(\vphi_i)\geq{}0}\vphi_i}$. This value is lower bounded by
\begin{align*}
&\En_{\vphi}\abs*{\sum_{i=1}^{d}\ind\crl*{\sgn(\vphi_i)\geq{}0}\vphi_i} \\
&= \En_{\vphi}\sum_{i=1}^{d}\ind\crl*{\sgn(\vphi_i)\geq{}0}\abs*{\vphi_i}
\intertext{Now, observe that since $N$ is odd we have $\sgn(\vphi_i)\in\pmo$, and so $\ind\crl*{\sgn(\vphi_i)\geq{}0}=(1+\sgn(\vphi_i))/2$. Furthermore, since $\vphi_i$ is symmetric, we may replace $\sgn(\vphi_i)$ with an independent Rademacher random variable $\sigma_i$}
&= \En_{\vphi}\En_{\sigma}\frac{1}{2}\sum_{i=1}^{d}(1+\sigma_i)\abs*{\vphi_i} \\
&= \En_{\vphi}\frac{1}{2}\sum_{i=1}^{d}\abs*{\vphi_i}.
\end{align*}
Lastly, the Khintchine inequality implies that $\En_{\vphi_i}\abs*{\vphi_i}\geq{}\sqrt{N/2}$, so the final lower bound is $\Omega(d\sqrt{N}) = \Omega(\sqrt{dn})$.

In the case where $d\geq{}n$, the argument above easily yields that $\En_{\eps}\sup_{w\in\cW}\nrm*{\sum_{t=1}^{n}\eps_{t}\grad{}\ls(w\midsem{}x_t,y_t)}_{2} = \Omega(n)$.

\end{proof}

\subsection{Proof of \pref{thm:relu_margin}}

\newcommand{\empxi}{\wh{\xi}_{n}}
\newcommand{\popxi}{\xi_{\cD}}

Before proceeding to the proof, let us introduce some auxiliary definitions and results. The following functions will be used throughout the proof. They are related by \pref{lem:phi_convergence}.
\[
\popxi(w,\gamma) = \En_{x \sim \cD}\ind\crl*{\tfrac{\abs*{\tri*{w,x}}}{\nrm*{w}\nrm*{x}}\leq{}\gamma},
\]
\[
\empxi(w,\gamma) = \frac{1}{n}\sum_{t=1}^{n}\ind\crl*{\tfrac{\abs*{\tri*{w,x_t}}}{\nrm*{w}\nrm*{x_t}}\leq{}\gamma}.
\]
\begin{lemma}
\label{lem:phi_convergence}
With probability at least $1-\delta$, simultaneously for all $w\in\cW$ and all $\gamma>0$,
\[
\popxi(w,\gamma) \leq{} \empxi(w,2\gamma) + \frac{4}{\gamma\sqrt{n}} + \sqrt{\frac{2\log\prn*{\log_2(4/\gamma)/\delta}}{n}},
\]
\[
\empxi(w,\gamma) \leq{} \popxi(w,2\gamma) + \frac{4}{\gamma\sqrt{n}} + \sqrt{\frac{2\log\prn*{\log_2(4/\gamma)/\delta}}{n}}.
\]

\end{lemma}
\begin{proof}[Proof sketch for \pref{lem:phi_convergence}]
We only sketch the proof here as it follows standard analysis (see Theorem 5 of \cite{kakade2009complexity}). The key technique is to introduce a Lipschitz function $\zeta_{\gamma}(t)$: 
$$
\zeta_{\gamma}(t) = \begin{cases} 
1 & |t| \leq \gamma \\
2 - |t|/\gamma & \gamma < |t| < 2\gamma \\
0 & |t| \geq 2\gamma
\end{cases}.
$$
Observe that $\zeta_{\gamma}$ satisfies $\ind\crl*{|t| > \gamma}\leq{}\zeta_{\gamma}(t)\leq{}\ind\crl*{|t| > 2\gamma}$ for all $t$. This sandwiching allows us to bound $\sup_{w\in\cW}\crl*{\popxi(w,\gamma) - \empxi(w,2\gamma)}$ (and $\sup_{w\in\cW}\crl*{\empxi(w,\gamma) - \popxi(w,2\gamma)}$ ) by instead bounding the difference between the empirical and population averages of the surrogate $\zeta_{\gamma}$. This is achieved easily using the Lipschitz contraction lemma for Rademacher complexity, and by noting that the Rademacher complexity of the class $\crl*{x\mapsto\tri*{w,x}\mid{}\nrm*{w}_{2}\leq{}1}$ is at most $\sqrt{n}$ whenever data satisfies $\nrm*{x_t}_{2}\leq{}1$ for all $t$. Finally, a union bound over values of $\gamma$ in the range $\brk*{0,1}$ yields the statement.
\end{proof}

\begin{proof}[\pfref{thm:relu_margin}] Let the margin function $\phi$ and $\delta>0$ be fixed. Define functions $\psi(\cdot)$, $\phi_1(\cdot)$, and $\phi_2(\cdot)$ as follows:
\begin{align*}
 \psi(\gamma) &= \frac{4}{\gamma\sqrt{n}} + \sqrt{\frac{2\log\prn*{\log_2(4/\gamma)/\delta}}{n}} \\
 \phi_1(\gamma) &= \phi(2\gamma) + \psi(\gamma) \\
 \phi_2(\gamma) &= \phi(4\gamma) + 2\psi(2\gamma).
\end{align*}
Now, conditioning on the events of \pref{lem:phi_convergence}, we have that with probability at least $1-\delta$,
\begin{align*}
\cW(\phi, \empdist) \subseteq{}\cW(\phi_1, \cD) \subseteq{}\cW(\phi_2, \empdist). \numberthis \label{eq:class_hierarchy}
\end{align*}

Consequently, we have the upper bound
\begin{align*}
\sup_{w\in\cW(\phi,\empdist)}\nrm*{ \nabla \poprisk(w) - \nabla \emprisk(w) }_2 
&\leq{} \sup_{w\in\cW(\phi_1,\cD)}\nrm*{ \nabla \poprisk(w) - \nabla \emprisk(w) }_2 \\
&\leq{} 
4\En_{\eps}\sup_{w\in\cW(\phi_1,\cD)}\nrm*{\frac{1}{n}\sum_{t=1}^{n}\eps_t\grad\ls(w\midsem{}x_t,y_t)} + 4\sqrt{\frac{\log(2/\delta)}{n}} \\
&\leq{} 
4\En_{\eps}\sup_{w\in\cW(\phi_2,\empdist)}\nrm*{\frac{1}{n}\sum_{t=1}^{n}\eps_t\grad\ls(w\midsem{}x_t,y_t)} + 4\sqrt{\frac{\log(2/\delta)}{n}}, \numberthis \label{eq:uniform_convergence_bound}
\end{align*}
where the second inequality holds with probability at least $1- \delta$ using \pref{lem:uniform_convergence}. They key here is that we are able to apply the standard symmetrization result because we have replaced $\cW(\phi,\empdist)$ with a set that does not depend on data. Next, invoking the chain rule (\pref{thm:chain_rule}), we split the Rademacher complexity term above as:
\begin{align}
\En_{\eps}\sup_{w\in\cW(\phi_2,\empdist)}\nrm*{\frac{1}{n}\sum_{t=1}^{n}\eps_t\grad\ls(w\midsem{}x_t,y_t)}
\leq{} 2 \underbrace{\En_{\eps}\sup_{w\in\cW(\phi_2,\empdist)}\frac{1}{n}\sum_{t=1}^{n}\eps_t\ind\crl*{y_t\tri*{w,x_t}\leq{}0}}_{(\blacklozenge)}
+\frac{2}{n} \En_{\eps}\nrm*{\sum_{t=1}^{n}\eps_tx_t}_{2}. \label{eq:rademacher_split}
\end{align}

The second term is controlled by \pref{thm:smooth_type}, which gives $\frac{1}{n} \En_{\eps}\nrm*{\sum_{t=1}^{n}\eps_tx_t}_{2}\leq{} \frac{1}{\sqrt{n}}$.  For the first term, we appeal to the fat-shattering dimension and the $\phi_2$-soft-margin assumption. 

\paragraph{Controlling $(\blacklozenge)$.}
Observe that for any fixed $\tgamma >0$, we can split $(\blacklozenge)$ as
\begin{align*}
\En_{\eps}\sup_{w\in\cW(\phi_2,\empdist)}\frac{1}{n}\sum_{t=1}^{n}\eps_t\ind\crl*{y_t\tri*{w,x_t}\leq{}0}
&\leq{}\underbrace{\En_{\eps}\sup_{w\in\cW(\phi_2,\empdist)}\frac{1}{n}\sum_{t=1}^{n}\eps_t\ind\crl*{y_t\tri*{w,x_t}\leq{}0\wedge\tfrac{\abs*{\tri*{w,x_t}}}{\nrm*{w}_2\nrm*{x_t}_2}\geq{}\tgamma}}_{(\star)} \\
&~~~~+\underbrace{\En_{\eps}\sup_{w\in\cW(\phi_2,\empdist)}\frac{1}{n}\sum_{t=1}^{n}\eps_t\ind\crl*{y_t\tri*{w,x_t}\leq{}0\wedge\tfrac{\abs*{\tri*{w,x_t}}}{\nrm*{w}_2\nrm*{x_t}_2}<\tgamma}}_{(\star\star)}.
\end{align*}
For $(\star\star)$, the definition of $\cW(\phi_2,\empdist)$ implies
\begin{align}
\En_{\eps}\sup_{w\in\cW(\phi_2,\empdist)}\frac{1}{n}\sum_{t=1}^{n}\eps_t\ind\crl*{y_t\tri*{w,x_t}\leq{}0\wedge\tfrac{\abs*{\tri*{w,x_t}}}{\nrm*{w}_2\nrm*{x_t}_2}<\tgamma}
\leq{} 
\sup_{w\in\cW(\phi_2,\empdist)}\frac{1}{n}\sum_{t=1}^{n}\ind\crl*{\tfrac{\abs*{\tri*{w,x_t}}}{\nrm*{w}_2\nrm*{x_t}_2}<\tgamma}
 \leq{} \phi_2(\tgamma). \label{eq:star_bound}
\end{align}

The quantity $(\star)$ can be bounded by writing it as
\[
\En_{\eps}\sup_{v\in{}V}\frac{1}{n}\sum_{t=1}^{n}\eps_tv_t,
\]
where V is a boolean concept class defined as $
V=\crl*{
\prn*{v_1(w), \ldots, v_n(w)}\in\crl*{0,1}^{n}\mid{}w\in\cW(\phi_2,\empdist)
}
$, where $v_i(w) \ldef \ind\crl*{y_i\tfrac{\tri*{w,x_i}}{\nrm*{w}_2\nrm*{x_i}_2}\leq{}0}\cdot \ind\crl*{\tfrac{\abs*{\tri*{w,x_i}}}{\nrm*{w}_2\nrm*{x_i}_2}\geq{}\tgamma}$. The standard Massart finite class lemma (e.g. \citep{mohri2012foundations}) implies
\[
\En_{\eps}\sup_{v\in{}V}\frac{1}{n}\sum_{t=1}^{n}\eps_tv_t \leq{} \sqrt{\frac{2\log\abs*{V}}{n}}.
\]
All that remains is to bound the cardinality of $V$. To this end, note that we can bound $\abs*{V}$ by first counting the number of realizations of  $\prn*{\ind\crl*{\tfrac{\abs*{\tri*{w,x_1}}}{\nrm*{w}_2\nrm*{x_1}_2}\geq{}\tgamma},\ldots,\ind\crl*{\tfrac{\abs*{\tri*{w,x_n}}}{\nrm*{w}_2\nrm*{x_n}_2}\geq{}\tgamma}}$ as we vary $w \in \cW(\phi_2,\empdist)$. This is at most ${n \choose n \phi_2(\tgamma)} \leq n^{n \phi_2(\tgamma)}$, since the number of points with margin smaller than $\tgamma$ is bounded by  $n \phi_2(\tgamma)$ via  \pref{eq:class_hierarchy}. 

Next, we consider only the points $x_t$ for which $\ind\crl*{\tfrac{\abs*{\tri*{w,x_t}}}{\nrm*{w}_2\nrm*{x_1}_2}\geq{}\tgamma} = 1$.
On these points, on which we are guaranteed to have a margin at least $\tgamma$, we count the number of realizations of $\ind\crl*{y_t\tfrac{\tri*{w,x_t}}{\nrm*{w}_2\nrm*{x_t}_2}\leq{}0}$. This is bounded by $n^{O\prn*{\frac{1}{\tgamma^2}}}$ due to the Sauer-Shelah lemma (e.g. \cite{shalev2014understanding}). The fat-shattering dimension at margin $\tgamma$ coincides with the notion of shattering on these points, and \cite{bartlett1999generalization} bound the fat-shattering dimension at scale $\tgamma$ by $O\prn*{\frac{1}{\tgamma^2}}$. Hence, the cardinality of $V$ is bounded by  
\begin{align}
|V| \le n^{n \phi_2(\tgamma)} n^{O\prn*{\frac{1}{\tgamma^2}}}.  \label{eq:V_bound}
\end{align}

\paragraph{Final bound.} Assembling equations \pref{eq:uniform_convergence_bound}, \pref{eq:rademacher_split}, \pref{eq:star_bound}, and \pref{eq:V_bound}  yields
\begin{align*}
\sup_{w\in\cW(\phi,\empdist)}\nrm*{ \nabla \poprisk(w) - \nabla \emprisk(w) }_2 & \le O\left(\phi_2(\tgamma) + \sqrt{\frac{\log|V|}{n}}+ \sqrt{\frac{\log(1/\delta)}{n}}\right)\\
& \le O\left(\phi_2(\tgamma) + \sqrt{ \prn*{\phi_2(\tgamma) + \frac{1}{\tgamma^2n}} \log(n)}+ \sqrt{\frac{\log(1/\delta)}{n}}\right)\\
& \le \tilde{O}\left(\sqrt{ {\phi_2(\tgamma)}} + \frac{1}{\tgamma \sqrt{n}}+ \sqrt{\frac{\log(1/\delta)}{n}}\right)\\
&\leq{} \tilde{O}\left(\sqrt{\phi(4\tgamma)} + \frac{1}{\tgamma}\sqrt{\frac{\log\prn*{1/\delta}}{n}} 
+ \frac{1}{\sqrt{\tgamma} n^{1/4}}\right).
\end{align*}
The chain of inequalities above follows by observing that $\phi_2(\tgamma) = \phi(4\tgamma) + 2 \psi(2\tgamma)$ is bounded and thus  $\phi_2(\gamma) \leq c \sqrt{\phi_2(\gamma)}$ for some constant $c$ independent of $\tgamma$. We get the desired result by optimizing over $\tgamma$. 
\end{proof}

\section{Additional Results}
\label{app:structural}

\begin{theorem}[Second-order chain rule for Rademacher complexity]
\label{thm:chain_rule_hessian}
Let two sequences of twice-differentiable functions $G_{t}:\bbR^{K}\to\bbR$ and $F_t:\bbR^{d}\to\bbR^{K}$ be given, and let $F_{t,i}(w)$ denote the $i$th of coordinate of $F_{t}(w)$. Suppose there are constants  $L_{F,1}$, $L_{F,2}$, $L_{G,1}$, $L_{G,2}$ such that for all $ 1 \leq t \leq n$, $\nrm*{\grad{}G_t}_{2}\leq{}L_{G,1}$, $\sqrt{\sum_{i,j}\nrm*{(\grad{}F_{t,i})(\grad{}F_{t,j})^{\trn}}_{\sigma}^{2}}\leq{}L_{F,1}$, $\nrm*{\grad^{2}G_t}_{2}\leq{}L_{G,2}$ and $\sqrt{\sum_{k=1}^{K}\nrm*{\grad^{2}F_{t,k}}_{\sigma}^{2}}\leq{}L_{F,2}$. Then, 
 
{\small 
\begin{align*}
\frac{1}{2}\En_{\eps}\sup_{w\in\cW}\nrm*{\sum_{t=1}^{n}\eps_{t}\grad{}^{2}(G_t(F_t(w)))}_{\sigma} &\leq L_{F,1}\En_{\beps}\sup_{w\in\cW}\sum_{t=1}^{n}\tri*{\beps_t,\grad{}G_t(F_t(w))}\\
& ~~~~+ L_{G,1}\En_{\beps}\sup_{w\in\cW}\nrm*{\sum_{t=1}^{n}\sum_{i=1}^{K}\beps_{t,k}\grad^2F_{t,i}(w)}_{\sigma} \\ 
& ~~~~+  L_{F,2}\En_{\tilde{\beps}}\sup_{w\in\cW}\sum_{t=1}^{n}\tri*{\tilde{\beps}_t,\grad^{2}G_t(F_t(w))} \\  
& ~~~~+  L_{G,2} \En_{\tilde{\beps}}\sup_{w\in\cW}\nrm*{\sum_{t=1}^{n}\sum_{i=1,j=1}^{K}\tilde{\beps}_{t,i,j}\grad{}F_{t,i}(w){\grad{}F_{t,j}(w)}^{\trn}}_{\sigma},
\end{align*}
}

where for all $i\in\brk*{K}$, $\nabla F_{t, i}(w)$ denotes the $i^{th}$ column of the Jacobian matrix $\nabla F_t \in \bbR^{d \times K}$,  $\nabla^2 F_{t, i}\in\bbR^{d\times{}d}$  denotes the $i$th slice of the Hessian operator $\nabla^2 F_t \in \bbR^{d \times d \times K}$, and $\beps \in \crl*{\pm 1}^{n, k}$ and $\tilde{\beps} \in \crl*{\pm 1}^{n \times K \times K}$ are matrices of  Rademacher random variables. 
\end{theorem}
As an application of \pref{thm:chain_rule_hessian}, we give a simple proof of dimension-independent Rademacher bound for the generalized linear model setting.
\begin{lemma}
\label{lem:glm_hessian}
Assume in addition to \pref{ass:glm} assume that $\abs*{\sigma'''(s)}\leq{}C_{\sigma}$ for all $s\in\cS$, and suppose $\nrm*{\cdot}$ is any $\beta$-smooth norm.
Then the empirical loss Hessian for the generalized linear model setting enjoys the normed Rademacher complexity bound,
\begin{equation}
\En_{\eps}\sup_{w\in\cW}\nrm*{\sum_{t=1}^{n}\eps_{t}\grad{}^{2}\ls(w\midsem{}x_t,y_t)}_{\sigma}
\leq{} O\prn*{\prn*{BR^{3}C_{\sigma}^{2}\sqrt{\beta} + C^2_{\sigma}R^2\sqrt{\log(d)}}\sqrt{n}}.
\end{equation}
\end{lemma}

It is easy to see that the same approach leads to a normed Rademacher complexity bound for the Hessian in the robust regression setting as well. We leave the proof as an exercise.
\begin{lemma}
\label{lem:rr_hessian}
Assume in addition to \pref{ass:rr} that $\abs*{\rho'''(s)}\leq{}C_{\rho}$ for all $s\in\cS$, and suppose $\nrm*{\cdot}$ is any $\beta$-smooth norm.
Then the empirical loss Hessian for the robust regression setting enjoys the normed Rademacher complexity bound:
\begin{equation}
\En_{\eps}\sup_{w\in\cW}\nrm*{\sum_{t=1}^{n}\eps_{t}\grad{}^{2}\ls(w\midsem{}x_t,y_t)}_{\sigma}
\leq{} O\prn*{\prn*{BR^{3}C_{\rho}\sqrt{\beta} + C_{\rho}R^2\sqrt{\log(d)}}\sqrt{n}}.
\end{equation}
\end{lemma}

\begin{proof}[\pfref{thm:chain_rule_hessian}]

We start by writing
\begin{equation}
\En_{\eps}\sup_{w\in\cW}\nrm*{\sum_{t=1}^{n}\eps_{t}\grad^2{}(G_t(F_t(w)))}_\sigma
= \En_{\eps}\sup_{w\in\cW}\sup_{\substack{u\in \bbR^d \\ \nrm*{u}_2 \leq{}1}}  
\sum_{t=1}^{n}\eps_{t}u^{\trn{}} \grad^2{}(G_t(F_t(w))) u. \label{eq:hessian_norm_vectorization}
\end{equation}

Using the chain rule for differentiation, we have for any $u \in \bbR^n$
{\small 
\begin{align*}
 u^{\trn{}} \grad^2{}(G_t(F_t(w))) u &= \tri*{\nabla F_t(w), u}^{\trn{}} ~\nabla^2 G_t(F_t(w))~ \tri*{\nabla F_t(w), u} + \tri*{\nabla G_t(F_t(w)), \nabla^2F_t(w)\brk*{u, u}},
\end{align*}
}
where $\nabla G_t(F_t(w))$ and $\nabla^2 G_t(F_t(w))$ denote the gradient and Hessian of $G_t$ at $F_t(w)$, and $\nabla^2F_t(w)\brk*{u, u} \in \bbR^K$ is a vector for which the $i$th coordinate is the evaluation of the Hessian operator for $F_{t,i}$ at  $(u, u)$. Using this identity along with \pref{eq:hessian_norm_vectorization}, we get
\begin{align*}
\En_{\eps}\sup_{w\in\cW}\nrm*{\sum_{t=1}^{n}\eps_{t}\grad^2{}(G_t(F_t(w)))}_\sigma &\leq  \En_{\eps}\sup_{w\in\cW}\sup_{\substack{u\in \bbR^d \\ \nrm*{u}_2 \leq{}1}}  
\sum_{t=1}^{n}\eps_{t} \tri*{\nabla G_t \prn*{F_t(w)}, \nabla^2F_t(w)\brk*{u, u}} \\
&~~~~+ \En_{\eps}\sup_{w\in\cW}\sup_{\substack{u\in \bbR^d \\ \nrm*{u}_2 \leq{}1}}  
\sum_{t=1}^{n}\eps_{t} \tri*{\nabla F_t(w), u}^{\trn{}} ~ \nabla^2 G_t(F_t(w)) ~ \tri*{\nabla F_t(w), u}.
\end{align*}

We bound the two terms separately.
\begin{enumerate}
\item \textit{First Term:}~ We introduce a new function that relabels the quantities in the expression. Let $h_1:  \bbR^{2K} \to \bbR$ be defined as $h_1(a, b) = \tri*{a, b}$, let $f_1:\cW \times \bbR^{d} \to\bbR^{K}$ be given by $f_{1}(w, u) = \grad{}G_t(F_t(w))$ and $f_{2}: \cW \times \bbR^{d} \to \bbR^K$ be given by $f_2(w, u) = (\grad^2{}F_{t,k}(w)\brk*{u, u})_{k\in\brk*{K}}$. We apply the block-wise contraction lemma \pref{lem:block_contraction} with one block for $f_1$ and one block for $f_2$ to conclude
\begin{align*}
&\frac{1}{2} \En_{\eps}\sup_{w\in\cW}\sup_{\substack{u\in \bbR^d \\ \nrm*{u}_2 \leq{}1}}  
\sum_{t=1}^{n}\eps_{t}h_1(f_1(w, u),f_2(w, u)) \\
&\hspace{0.5in} \leq{}   L_{F, 1}  \En_{\eps} \sup_{w\in\cW}\sup_{\substack{u\in \bbR^d \\ \nrm*{u}_2 \leq{}1}}  
\sum_{t=1}^{n}\tri*{\beps_{t}, f_1(w, u)}  +  L_{G, 1} \En_{\eps}\sup_{w\in\cW}\sup_{\substack{u\in \bbR^d \\ \nrm*{u}_2 \leq{}1}}  
\sum_{t=1}^{n}\tri*{\beps_{t}, f_2(w, u)} \\
&\hspace{0.5in} \leq{}  L_{F, 1}  \En_{\eps} \sup_{w\in\cW}\sup_{\substack{u\in \bbR^d \\ \nrm*{u}_2 \leq{}1}}  
\sum_{t=1}^{n}\tri*{\beps_{t}, \grad{}G_t(F_t(w))}  +  L_{G, 1} \En_{\eps}\sup_{w\in\cW}\sup_{\substack{u\in \bbR^d \\ \nrm*{u}_2 \leq{}1}}  
\sum_{t=1}^{n} \tri*{\beps_{t},  \grad^2 F_t(w) \brk*{u, u}  }\\
&\hspace{0.5in} \leq{}  L_{F, 1}  \En_{\eps} \sup_{w\in\cW} 
\sum_{t=1}^{n}\tri*{\beps_{t}, \grad{}G_t(F_t(w))}  +  L_{G, 1} \En_{\eps}\sup_{w\in\cW}\sup_{\substack{u\in \bbR^d \\ \nrm*{u}_2 \leq{}1}}  
\prn*{\sum_{t=1}^{n} \grad^2 F_t(w) \beps_t  } \brk*{u, u}\\
&\hspace{0.5in} =  L_{F, 1}  \En_{\eps} \sup_{w\in\cW}
\sum_{t=1}^{n}\tri*{\beps_{t}, \grad{}G_t(F_t(w))}  + L_{G, 1} \En_{\eps}\sup_{w\in\cW}
\nrm*{\sum_{t=1}^{n} \grad^2 F_t(w) \beps_t  }_\sigma.
\end{align*}
\item \textit{Second Term:}~Let us first simplify as
\begin{align*}
\tri*{\nabla F_t(w), u}^{\trn{}} ~ \nabla^2 G_t(F_t(w)) ~ \tri*{\nabla F_t(w), u} &= \sum_{i,j=1}^K \tri*{\nabla F_t(w), u}_i \nabla^2 G_t(F_t(w))_{i, j} \tri*{\nabla F_t(w), u}_j \\
&= \sum_{i,j=1}^K \prn*{u^{\trn{}} \nabla F_{t, i}(w)} \times \frac{\partial^2 G_t}{\partial z_i \partial z_j} \times  \prn*{\nabla F_{t, j}(w)^{\trn{}} u} \\
&= \sum_{i,j=1}^K \frac{\partial^2 G_t}{\partial z_i \partial z_j}  \prn*{ u^{\trn{}}  { \nabla F_{t, i}(w)\nabla F_{t, j}(w)^{\trn{}}} u}, 
\end{align*}

where $\nabla F_{t, j}(w) \ldef{} \nabla F_{t}(w)[:, j] \in \bbR^d$, and the last equality follows by observing that $\frac{\partial^2 G_t}{\partial z_i \partial z_j}$ is scalar. We thus have
\begin{align*}
& \En_{\eps}\sup_{w\in\cW}\sup_{\substack{u\in \bbR^d \\ \nrm*{u}_2 \leq{}1}}  
\sum_{t=1}^{n}\eps_{t} \tri*{\nabla F_t(w), u}^{\trn{}} ~ \nabla^2 G_t(F_t(w))~ \tri*{\nabla F_t(w), u} \\ & \hspace{1in}= \En_{\eps}\sup_{w\in\cW}\sup_{\substack{u\in \bbR^d \\ \nrm*{u}_2 \leq{}1}}  
\sum_{t=1}^{n}\eps_{t} \sum_{i,j=1}^K \frac{\partial^2 G_t}{\partial z_i \partial z_j}  \prn*{ u^{\trn{}}  { \nabla F_{t, i}(w)\nabla F_{t, j}(w)^{\trn{}}} u}.
\end{align*}

Similar to the first part, we introduce a new function that relabels the quantities in the expression. Let $h_2:  \bbR^{2 K^2} \to \bbR$ be defined as $h_2(a, b) = \sum_{i, j=1}^K a_{i, j} b_{i, j} $. Let $f_1:\cW \times \bbR^{d} \to\bbR^{K^2}$ be given by $f_{1}(w, u) = (\grad^2 G_t)(F_t(w))$ and $f_{2}:\cW \times \bbR^{d} \to\bbR^{K^2}$ be given by $f_2(w, u) =  \prn*{ u^{\trn{}}  \nabla F_{t, i}(w)\nabla F_{t, j}(w)^{\trn{}} u}_{\substack{i, j \in [K]}}$. We apply block-wise contraction (\pref{lem:block_contraction}) with one block for $f_1$ and one block for $f_2$ to conclude
\begin{align*}
&\frac{1}{2} \En_{\eps}\sup_{w\in\cW}\sup_{\substack{u\in \bbR^d \\ \nrm*{u}_2 \leq{}1}}  
\sum_{t=1}^{n}\eps_{t}h_2(f_1(w, u),f_2(w, u)) \\
&\hspace{0.5in} \leq{}   L_{F, 2}  \En_{\eps} \sup_{w\in\cW}\sup_{\substack{u\in \bbR^d \\ \nrm*{u}_2 \leq{}1}}  
\sum_{t=1}^{n}\tri*{\beps_{t}, f_1(w, u)}  +  L_{G, 2} \En_{\eps}\sup_{w\in\cW}\sup_{\substack{u\in \bbR^d \\ \nrm*{u}_2 \leq{}1}}  
\sum_{t=1}^{n}\tri*{\beps_{t}, f_2(w, u)} \\
 &\hspace{0.5in} \leq L_{F, 2}\En_{\eps}\sup_{w\in\cW}
\sum_{t=1}^{n} \tri*{\beps_{t} , \nabla^2 G_t(F_t(w))} +  L_{G,2} \En_{\eps}\sup_{w\in\cW} \sup_{\substack{u\in \bbR^d \\ \nrm*{u}_2 \leq{}1}}  
\sum_{t=1}^{n} \sum_{i, j=1}^K \beps_{t, i, j} u^{\trn{}}  \nabla F_{t, i}(w)\nabla F_{t, j}(w)^{\trn{}} u  \\
 &\hspace{0.5in} = L_{F, 2}\En_{\eps}\sup_{w\in\cW}
\sum_{t=1}^{n} \tri*{\beps_{t} , \nabla^2 G_t(F_t(w))} +  L_{G,2} \En_{\eps}\sup_{w\in\cW}
\nrm*{\sum_{t=1}^{n} \sum_{i, j=1}^K \beps_{t, i, j}   \nabla F_{t, i}(w)\nabla F_{t, j}(w)^{\trn{}} }_\sigma. 
\end{align*}
\end{enumerate}

Combining the two terms gives the desired chain rule. 
\end{proof}

\begin{proof}[\pfref{lem:glm_hessian}]
As in \pref{lem:glm_gradient}, let $G_{t}(s) = \prn*{\sigma(s)-y_t}^{2}$ and $F_{t}(w)=\tri*{w,x_t}$, so that $\ls(w\midsem{}x_t,y_t)=G_t(F_t(w))$.

Observe that $G'_t(s)=2\prn*{\sigma(s)-y_t}\sigma'(s)$, $\grad{}F_{t}(w)=x_t$, $\grad^{2}F_{t}=\mb{0}$, $G''_{t}(s)(s)=2(\sigma'(s))^{2} + 2y_t\sigma''(s)$, and $G'''(s)=4\sigma'(s)\sigma''(s) + 2y_t\sigma'''(s)$, which implies that $\abs*{G'''_t(s)}\leq{}6C_{\sigma}^{2}$. Using \pref{thm:chain_rule_hessian} with constants $L_{F,1 } = R^2, ~ L_{F, 2} = 0,~ L_{G, 1} = 2C_\sigma^2~ \text{and} ~ L_{G, 2} = 4 C_\sigma^2 $, we get
\begin{align*}
\En_{\eps}\sup_{w\in\cW}\nrm*{\sum_{t=1}^{n}\eps_{t}\grad{}^{2}\ls(w\midsem{}x_t,y_t)}_{\sigma}
&\leq{}  2 R^2 \En_{\beps}\sup_{w\in\cW}\sum_{t=1}^{n}\eps_t G''_t(\tri*{w, x_t}) +  8C_\sigma^2 \En_{{\eps}}\sup_{w\in\cW}\nrm*{\sum_{t=1}^{n}{\eps}_{t}x_t x_t^{\trn{}}}_{\sigma} , 
\intertext{applying \pref{lem:block_contraction}, }
&\leq{}  12 R^2 C_\sigma^2 \En_{\eps}\sup_{w\in\cW}\sum_{t=1}^{n}\eps_t \tri*{w, x_t} +  8C_\sigma^2 \En_{{\eps}}\sup_{w\in\cW}\nrm*{\sum_{t=1}^{n}{\eps}_{t}x_t x_t^{\trn{}}}_{\sigma}  \\
&= 12R^{2}C_{\sigma}^{2}B\En_{\eps}\nrm*{\sum_{t=1}^{n}\eps_tx_t} + 8C^2_{\sigma}\En_{\eps}\nrm*{\sum_{t=1}^{n}\eps_tx_tx_t^{\trn}}_{\sigma}.
\end{align*}
Invoking \pref{thm:smooth_type} and \pref{fact:smooth}, we have $\En_{\eps}\nrm*{\sum_{t=1}^{n}\eps_tx_t}\leq{}\sqrt{2\beta{}R^{2}n}$ and $\En_{\eps}\nrm*{\sum_{t=1}^{n}\eps_tx_tx_t^{\trn}}_{\sigma}\leq{}\sqrt{2\log(d)R^{4}n}$.
\end{proof}

\end{document}